\documentclass[a4paper,11pt]{article}

\usepackage{amsmath,amsfonts,amssymb,caption,graphicx}
\usepackage{ushort}
\usepackage{makeidx}
\usepackage{float}
\usepackage{accents}
\usepackage{color}
\usepackage{framed}
\usepackage{versions}
\usepackage{xcolor}
\usepackage{emptypage}
\usepackage{array}
\usepackage{multirow}
 \usepackage{amsthm}
\usepackage{subcaption}

\usepackage{amsthm}

\usepackage{wrapfig}
\usepackage{booktabs}       
\def\hmath$#1${\texorpdfstring{{\rmfamily\textit{#1}}}{#1}}

\usepackage{tikz}
\usetikzlibrary{trees}
\newtheorem{theorem}{Theorem}
\newtheorem{lemma}{Lemma}
\newtheorem{corollary}{Corollary}
\newtheorem{proposition}{Proposition}
\usepackage[bookmarks]{hyperref}
\hypersetup{
    colorlinks=true,   	
    linkcolor=black,      
    citecolor = [rgb]{0 0.7 0},   	
    filecolor=magenta, 	
    urlcolor=blue
}
 
\input{epsf}

\DeclareMathOperator{\sign}{sign}

\newcommand{\Prob}{\mathbb{P}}
\newcommand{\R}{\mathbb{R}}

\newcommand{\Nopt}{N^*_{K,d}}

\newcommand{\normal}{\mathcal{N}}

\newtheorem{ques}[]{Question}

\usepackage{pgfplots}

\usepackage{xurl}

\title{Fundamental limits of distributed multiclass classification from simple binary decisions}

 \author
 {
 	 Ioannis Papageorgiou
     \thanks{ School of Mathematics, University of Bristol.
                 Email: \texttt{\href{tf24342@bristol.ac.uk}%
			 {tf24342@bristol.ac.uk}}.
 	 }
  \and
     Srinivas Nomula
     \thanks{ Department of Electrical Communication
Engineering, Indian Institute of Science.
                 Email: \texttt{\href{sivasrinivas@iisc.ac.in}%
			 { sivasrinivas@iisc.ac.in}}.
         }
  \and
     Ayalvadi Ganesh 
     \thanks{School of Mathematics, University of Bristol.
                 Email: \texttt{\href{A.Ganesh@bristol.ac.uk}%
			 { A.Ganesh@bristol.ac.uk}}.
         } 
  \and
\mbox{}\\[-2ex]       
  \and 
      Sidharth Jaggi
     \thanks{School of Mathematics, University of Bristol.
                 Email: \texttt{\href{sid.jaggi@bristol.ac.uk}%
			 { sid.jaggi@bristol.ac.uk}}.
         }    
  \and
     Parimal Parag
     \thanks{Department of Electrical Communication
Engineering, Indian Institute of Science.
                 Email: \texttt{\href{parimal@iisc.ac.in}%
			 { parimal@iisc.ac.in}}.
         }    
 }

\begin{document}
\maketitle

\begin{center}
    \large \textbf{Abstract}
\end{center}
We consider the problem of constructing a $K$-class classifier from the combination of $O(\log K)$ simple binary classifiers -- this is a natural paradigm to construct a sophisticated classifier in a distributed manner with each agent performing a relatively straightforward task. We study the fundamental performance limits of such a classifier when the corresponding binary classifiers are hyperplanes. 
For a stylized Gaussian setting where the $K$ class centers are independent Gaussian points in $\R^d$ and the observations are 
corrupted by Gaussian noise, we derive explicit performance bounds across several decoding and dimensional regimes.
Extensive simulation experiments provide strong empirical validation of the presented theoretical results.

\bigskip


\newpage

\section{Introduction}

Classification is one of the foundational tasks to which machine learning is applied, with uses ranging from image recognition to spam and fraud detection to medical diagnostics. The typical approach to building classifiers is to gather all the relevant data and use it to train a single complex and powerful classifier. In this paper, we study whether a large number of simple classifiers, each trained on a subset of the data, could cooperate to perform classification with an accuracy comparable to that of a single complex classifier that has access to all the data. There are multiple motivations for this work, including improved fault tolerance and data privacy.

Two complementary perspectives form the starting point of this work: multiclass-to-binary
reduction and constrained distributed inference. Many classical supervised
learning methods were originally developed for binary classification, and a
common approach to multiclass problems is therefore to decompose a \(K\)-class
decision into a collection of binary decisions. Standard examples include
one-versus-all schemes, which train one binary classifier per class, and
one-versus-one schemes, which train a binary classifier for each pair of classes~\cite{hastie1997pairwise,aly2005survey}. 

\smallskip

\textbf{Error-correcting output codes. } More generally, error-correcting
output codes (ECOC) provide a coding-theoretic framework for multiclass classification:
each class is assigned a binary codeword, a collection of binary classifiers is
evaluated, and the final prediction is obtained by decoding the resulting binary
output vector to one of the class codewords
\cite{dietterich1995ecoc,allwein2000reducing}.
In classical ECOC methods, the code matrix
specifies how the classes are grouped for each binary classifier. Later
extensions allow more general coding schemes, including ternary codes in which
some classes may be ignored for a given binary classifier, random dense and
sparse code designs \cite{allwein2000reducing}, and data-dependent variants
that adapt the encoding to the classification problem~\cite{pujol2006discriminant,escalera2008subclass,pujol2008incremental}; see
\cite{escalera2010ecoc} for an
implementation and summary of several important ECOC variants.
Recent work has continued to revisit ECOC from modern perspectives, including
scalable codebook design, deep ECOC architectures, and robustness-oriented
output-code learning
\cite{gupta2022scalable,yu2024ecocweight,chou2025contrastive}.

Our setting follows the same high-level principle, but it can be viewed as a random geometric analogue of ECOC. Rather than appropriately designing a binary decomposition (using any of the above methods), and training a separate binary classifier for each resulting binary partition, the binary tests we use are simple randomly chosen hyperplanes. This leads to a different set of questions. In particular, rather than asking how
to design or learn an effective binary decomposition, we investigate the fundamental performance limits of a very simple and easily-distributed classification scheme, under noise and
communication constraints.

The key outcome of our analysis is that this random geometric construction is
not merely a mathematically convenient variant of ECOC. The theory suggests that for Gaussian mixture models,
random hyperplane codes can outperform commonly used ECOC decompositions like
the random subsets \cite{allwein2000reducing}, across a wide range of parameter
regimes. Moreover, in this setting their performance often approaches that of optimally chosen
hyperplanes, despite requiring no data-dependent optimization of the binary
tests. This indicates that random geometric measurements can provide a
surprisingly effective coding mechanism for multiclass classification. The use of multiple binary tests introduces redundancy, so that classification
can remain reliable even when some individual binary decisions are corrupted by
noise. In this sense, the framework combines the simplicity of binary
classification with the robustness of code-based decoding.

\smallskip

\textbf{Distributed learning.} A second motivation comes from constrained and distributed inference. Modern
classification systems often operate under practical limitations imposed by
communication bandwidth, memory, energy consumption, latency, or privacy. In
sensor networks, edge devices, and other distributed environments, it may be
impractical to transmit full-precision measurements or feature vectors to a
central processor. A natural alternative is to allow each local device to perform
a simple binary decision and communicate only a highly compressed message to a central server.
This perspective differs from the standard federated-learning paradigm, where
distributed devices collaborate during training by communicating gradients or
model updates to the central server
\cite{mcmahan2017fedavg,kairouz2021advances}. In our approach, rather than distributing the training
algorithm, we distribute the decision process itself. The central server receives
a collection of local binary responses and combines them to make a multiclass
prediction.

Beyond its practical motivation, this architecture provides a clean theoretical
lens for studying classification from compressed information. The paper focuses
on a deliberately simple random geometric model, in which class prototypes are
drawn from a Gaussian distribution and binary responses are generated by random
hyperplane tests. This abstraction is not intended to capture all aspects of
modern datasets or learned classifiers. Instead, it isolates the basic
phenomena of class separation, noise, compression, and decoding in a setting
where a complete analysis is possible. The resulting theory gives fundamental
limits for this stylized classification problem and provides a baseline for what
one may expect in more realistic classification tasks involving richer data
models and more flexible decision rules.

\smallskip

\textbf{Further connections and comments. }
Our framework also connects naturally to decentralized detection
and distributed hypothesis testing~\cite{tenney1981distributed,tsitsiklis1988decentralized,varshney1997distributed,ahlswede2003hypothesis}. In those settings, multiple sensors observe
data locally and transmit compressed messages to a central server, which decides
among competing hypotheses.
Our construction can be viewed as a multiclass classification analogue of that, where binary local decisions are combined by a decoder. 
In this perspective, different
decoding rules correspond to different ways of combining local binary evidence,
ranging from simple majority or distance-based rules to likelihood-based decisions.

Finally, a related area is one-bit compressed sensing and quantized signal
processing, where signals are observed only through severely quantized linear
measurements
\cite{boufounos2008onebit,jacques2013robust,plan2013onebit}. However, rather than recovering an unknown signal
from these measurements, the objective here is to study how such binary information can be used to
distinguish among a finite collection of candidate classes. This shifts the focus from recovery to classification and
leads naturally to questions about decoding, error probability, and the
reliability of binary decision systems.

\vspace*{-0.15 cm}

\subsection{Outline of contributions}

\vspace*{-0.05 cm}

Motivated as above, in this work we investigate the fundamental performance limits of a $K$-class classifier that is constructed as a combination of $N$ simple binary classifiers. 
We introduce a simple classification algorithm where the base classifiers associated to each agent are simple randomly generated hyperplanes. 
The first  part of the paper is focused on a Gaussian noiseless setting, where the class centers are independent Gaussian points in $\mathbb R ^d$. We provide sharp theoretical analysis (upper and lower bounds) for the performance of a)~the proposed random hyperplane scheme b)~natural ECOC alternatives and c)~information-theoretic optimal schemes. We conclude that the proposed scheme is of much higher practical interest than a standard random-subsets ECOC alternative~\cite{allwein2000reducing}, and actually show that in sufficiently high dimensions its performance is essentially similar to the information-theoretic limits. 
In the second part of the paper, noise is introduced, so that observations are obtained as one of the class centers corrupted by Gaussian noise. We derive explicit performance bounds for the random hyperplanes scheme across several decoding and dimensional regimes, and propose a computationally efficient decoder that performs nearly optimally in most regimes of practical interest.

The rest of this paper is organised as follows.
Section~\ref{s:noiseless} is focused on the noiseless setting, in which we describe the random hyperplanes algorithm and the random subsets baseline~\cite{allwein2000reducing}, and derive our main theoretical results for both schemes. Specifically, for random hyperplanes
we prove matching upper and lower bounds showing that, in
sufficiently high dimensions, 
\[
N = 2\log_2 K+\log_2(1/\delta),
\]
random hyperplanes are sufficient to assign distinct codewords to all \(K\)
class centers with probability at least \(1-\delta\). Since any scheme using
\(N\) binary tests can distinguish at most \(2^N\) classes, at least
\(\log_2 K\) tests are necessary for any choice of hyperplanes. Thus, the proposed scheme is within a
factor of two of the information-theoretic optimum, despite requiring no
data-dependent optimization; something  perhaps surprising having in mind the simplicity of the scheme. 

For the random subsets scheme, we also show that in sufficiently high-dimensions the same number of $N=2\log_2 K+\log_2(1/\delta)$ hyperplanes is sufficient, but the dimensional requirement is much stronger in this case. In particular, random hyperplanes require a dimension only
logarithmic in $K$, whereas random subsets require a
dimension linear in $K$. This shows that the proposed scheme is substantially more dimension-efficient, something that is also verified in practice, in a number of simulation experiments.

We then focus on the more realistic practical setting where observations are corrupted by Gaussian noise. In Section~\ref{s:orth_hyps}, as a tractable intermediate step towards understanding the performance of random hyperplane classifiers in the presence of noise, we first consider a simpler setting where hyperplanes are chosen to be orthogonal to each other.  We derive explicit
non-asymptotic bounds and exact finite-sample expressions for both Hamming and
maximum likelihood~(ML) decoding. Importantly, these results indicate that $N = O (\log K )$ hyperplanes are still sufficient, but now including a multiplicative constant that depends on the signal to noise ratio~(SNR). In specific, we derive
achievability conditions of
the form,
\[
N
\ge
\frac{\log((K-1)/\delta)}{E(\rho)},
\]
sufficient  to ensure an error probability smaller than $\delta$,
where \(\rho\) denotes the SNR and \(E(\rho)\) is the relevant
decoder-dependent error exponent.

In Section~\ref{s:random_hyps}, the analysis is ultimately extended to random hyperplanes. We obtain an
explicit finite-dimensional Chernoff bound for Hamming decoding, characterize
its limiting probability of error, and establish high-dimensional regimes in
which the random-hyperplane model approaches its orthogonal counterpart.
Furthermore, due to the fact that exact ML decoding for random hyperplanes turns out to be particularly expensive, we also introduce a computationally efficient approximate ML decoder for this setting, which is motivated by the earlier orthogonal hyperplane analysis.
We formally prove that this decoder is asymptotically optimal in
sufficiently high dimensions, and provide extensive simulation results indicating that in practice it actually performs nearly as ML even in much lower dimensions. In fact, a dimension roughly at the same order as the number of hyperplanes seems to be enough.

Finally, the experimental results lead to an important conclusion for general classification problems. As perhaps expected, Hamming decoding is found to perform substantially worse than the proposed reliability-aware decoder. This 
indicates that  assigning equal weight to the binary decisions associated with each classifier/agent can be significantly suboptimal, and
highlights the importance of exploiting heterogeneous classifier reliabilities.

\section{Noiseless setting} \label{s:noiseless}

Throughout this paper, the notation we use is that the number of classes is $K$, the dimensionality of the data is $d$ (which is typically considered to be large),  and there are $N$ agents/binary classifiers who seek to cooperate to solve the classification problem.
In this section, we first focus on a noiseless Gaussian setting, where the $K$ class centers, $x_1,..., x_K$, are independent and identically distributed (i.i.d.) Gaussian points in $\mathbb R ^ d$. In specific, $x_1,\dots,x_K  \stackrel{\mathrm{iid}}{\sim}\mathcal N(0,I_d)$, where $I_d$ denotes the~$d$-dimensional identity matrix. We first consider the fundamental task of separating the class centers themselves. In order to do so, we consider two algorithms, which we refer to as the random hyperplanes and random subsets algorithms.

\subsection{Algorithms}\label{s:algs}

\noindent{\bf Random hyperplanes algorithm.} Each of the $N$ agents picks a random hyperplane in $\R^d$ passing through the origin, as follows. Agent $i$ picks a unit vector $W_i$ uniformly at random from~$\mathcal{S}^{d-1}$, the unit sphere in $\R^d$, independent of other agents. The vectors $w_i$ are the corresponding normal vectors that define the hyperplanes $H_i= \{ x\in \R^d: w_i ^ \top x = 0\}$. Now, with each class $k$, we associate an $N$-bit codeword $C_k \in (\pm 1)^N$, whose bits are given by,
\begin{equation}
\label{eq:code_hyper}
C_{k,i} :=  \sign( w_i ^ \top x_k  ), \quad i=1,\ldots,N.
\end{equation} 
If the $K$ codewords are all distinct, then the agents can collectively identify the classes perfectly. 
\begin{ques}
For the random hyperplanes algorithm, how large should $N$ be in order to guarantee that all $K$ codewords are distinct with probability at least $1-\delta$, where $\delta>0$ is a specified error~tolerance?   
\end{ques}

It is obvious by the pigeonhole principle that we need $N\geq \log_2 K$. We show that we actually need at least $2\log_2 K$ and that, if $d$ is sufficiently large, then do not need much more.

\medskip

\noindent{\bf Random subsets algorithm.} Each of the $N$ agents randomly partitions the set of classes into two subsets of size approximately $K/2$ by Bernoulli sampling of the class labels. Denote the subsets corresponding to agent $i$ by $A_i^+$ and $A_i^-$, and let $\chi_i$ be the indicator that the subsets $A_i^+$ and $A_i^-$ are linearly separable, i.e., that they are separated by an affine hyperplane (one not necessarily passing through the origin). In this setting,
only agents $i$ that can linearly separate their subsets participate in discriminating between classes; otherwise, they assign the same symbol, $-1$, to all classes.
So, we associate a codeword $\tilde{C}_k$ with class $k$ as follows, 
\begin{equation}
\label{eq:code_subsets} 
\tilde{C}_{k,i} :=\begin{cases}
+1, &\mbox{ if } \chi_i=1 \mbox{ and } k\in A_i^+, \\
-1, &\mbox{ otherwise.}
\end{cases}
\end{equation} 
Again, the agents can collectively identify the classes only if all $K$ codewords are distinct. 
\begin{ques}
For the random subsets algorithm, how large should $N$ be in order to guarantee that all $K$ codewords are distinct with probability at least $1-\delta$, where $\delta>0$ is a specified error tolerance?
\end{ques}

We now state the answers to these questions, providing necessary and sufficient conditions on the number of agents (and hence hyperplanes) required for each scheme. In the following, $E$ refers to the error event that at least two distinct classes share the same codeword.

\subsection{Theoretical results}

\begin{theorem}
\label{thm:hyp_suff}
Let $x_1,\dots,x_K  \stackrel{\mathrm{iid}}{\sim}\mathcal N(0,I_d)$, and 
$w_1,\dots,w_N $ be i.i.d. spherically symmetric random vectors defining $N$ hyperplanes in $\R^d$. Let $\delta \in (0,1)$ and suppose $d\ge C \log K$ for a suitable constant $C>0$. Then, in order to ensure that the error probability satisfies $\Pr (E) \le \delta$, it is sufficient that the number of hyperplanes satisfies,
\[
N \; \ge\; \Big(2\log_2 K + \log_2(1/\delta)\Big)
\left( 1 + O \left(\sqrt{\tfrac{\log(K^2/\delta)}{d}}\right) \right).
\]
\end{theorem}

\noindent \emph{Remarks.} In the high-dimensional regime $d \gg \log K$, the correction term vanishes and the theorem says that $2\log_2 K + \log_2(1/\delta)$ hyperplanes are sufficient to separate all the classes, with high probability. A more careful analysis shows that, in the intermediate regime where $d$ is of the same order as $\log K$, the required number of hyperplanes is still $O(\log K)$, with constants depending on the ratio between $d$ and $\log K$. In even lower dimensions, a logarithmic number of hyperplanes is not enough, and in particular, in the fixed-dimensional case (i.e., when $d$ is a constant), a polynomial number of hyperplanes is needed; see proof below for more details.

\vspace*{-0.05 cm}

\begin{theorem}
\label{thm:hyp_necc}
Let $x_1,\dots,x_K$ and $w_1,\dots,w_N$ be as in Theorem~\ref{thm:hyp_suff}. Then, in order to ensure that the error probability satisfies $\Pr (E) \le \delta$, it is necessary that,
\[
N \;\ge\; 2\log_2 K \;+\; \log_2(1/\delta) \;-\; C,
\]
where C is an absolute constant; e.g., one may take $C=4$.
\end{theorem}

\noindent \emph{Remark.} In contrast to the upper bound on $N$ in Theorem~\ref{thm:hyp_suff}, this lower bound is dimension-independent.

\vspace*{-0.1 cm}

\begin{theorem}
\label{thm:subsets_suff}
Let $x_1,\dots,x_K  \stackrel{\mathrm{iid}}{\sim}\mathcal N(0,I_d)$. Generate $N$ random subset-pairs using i.i.d. $\mathrm{Bern}(1/2)$ labelings and retain only those that are separable by an affine hyperplane. 

\medskip

\noindent \textbf{a) High-dimensional regime.} 
If $d \ge (1+\varepsilon)\tfrac{K-1}{2}$ for some $\varepsilon>0$, and if 
$\rho _K $ is an exponentially small correction term,
$$
\rho _K = 
\exp\left(
-\frac{\varepsilon^2}{4+2\varepsilon}(K-1)
\right),
$$
then  in order to ensure that 
$
\Pr(E)
\le
\delta
+
O\left(
\rho _K
\log K
\right )
$, it is sufficient that the total number of random labelings satisfies,
\[
N
\ge
\left(
2\log_2 K+\log_2(1/\delta)
\right)
\left(
1+
O\left(
\rho _K
\right)
\right).
\]

\smallskip

\noindent \textbf{b) Low-dimensional regime.} 
If $d \le (1-\varepsilon)\tfrac{K-1}{2}$ for some $\varepsilon>0$, then exponentially  many trials  (in $K$) are needed. So, the random subset scheme is infeasible in this case.
\end{theorem}

\noindent \emph{Remark.} Comparing Theorems~\ref{thm:hyp_suff} and \ref{thm:subsets_suff}, we see that the latter needs the data to be of much higher dimension, of order $K$ rather than $\log K$, in order to be classifiable using a small number of hyperplanes. However, if the condition on dimension is met, then  both schemes need essentially the same number of hyperplanes, $N = 2\log_2K + \log_2(1/\delta)$.

\begin{theorem} 
\label{thm:opt_number_hyp}
Let $x_1,\dots,x_K \stackrel{\mathrm{iid}}{\sim}\mathcal N(0,I_d)$, and let $\Nopt$ denote the minimum number of affine hyperplanes required to separate the $K$ points. The following lower bounds hold:
\begin{enumerate}
\item 
Irrespective of the number of dimensions,
\[
\Nopt \ge  \lceil \log_2 K \rceil .
\]
\item 
If $d < \tfrac{1}{2}\log_2 K$, then,
\[
\Nopt  \ge \frac{1}{2e}\, d K^{1/d}.
\]
\end{enumerate}
\end{theorem}

\noindent \emph{Remark.} In the
high-dimensional regime $d\gg\log K$, the number of random hyperplanes required is within a
factor of two of the information-theoretic minimum $\Nopt$, despite requiring no
data-dependent optimization. In the low dimension case where $d$ is fixed, the number of random hyperplanes and the optimal $\Nopt$ both grow polynomially in~$K$.

\subsection{Proofs}

\begin{proof}[Proof of Theorem~\ref{thm:hyp_suff}] 
First, we revisit the concentration of pairwise angles among Gaussian
class centers. For \(K\) i.i.d. Gaussian points in
\(\mathbb R^d\), we show that these angles concentrate around \(\pi/2\) with high
probability, and quantifying this concentration in our setting allows us to control
the collision probability of the random hyperplane code.

Define 
$\hat x_i := x_i/\|x_i\|$, so that $\hat x_i$, $i=1,\ldots,K$, are i.i.d. uniform on the sphere $S^{d-1}$. By standard sub-Gaussian concentration of inner products for uniform distributions on the sphere~\cite{vershynin2018high}, we have for any $t \in (0,1)$ that,
\[
\Pr\big(|\langle \hat x_i,\hat x_j\rangle|\ge t\big) \;\le\; 2\,e^{-c\,d\,t^2},
\]
for some universal constant $c>0$; one may take $c=1/4$ in a standard version of this bound. By applying a union bound over the $\binom{K}{2} \le \tfrac{K^2}{2}$ pairs, we obtain with probability at least $1-\delta/2$~that,
\[
\max_{i<j}|\langle \hat x_i,\hat x_j\rangle|
\;\le\;
t_\star := \sqrt{\tfrac{c'\,\log(K^2/\delta)}{d}},
\]
for a universal constant $c'>0$.  
Writing $\theta_{ij}$ for the pairwise angle between $x_i$ and $x_j$, we get that with high probability,
\[
\max_{i<j}| \cos \theta _{ij}|
\;\le\;
t_\star := \sqrt{\tfrac{c'\,\log(K ^ 2/\delta)}{d}}.
\]
Since $\delta$ is fixed, if \(d\ge C' \log K\) for a sufficiently large constant $C ' >0$, as in the assumptions of the theorem, we get \(t_\star\le 1\), and therefore, with probability at least $1-\delta/2$, 
\begin{equation} \label{dist}
|\theta_{ij}-\pi/2|
\le \arcsin(t_\star)
\le \frac{\pi}{2}t_\star
\le C  \frac{\log(K ^ 2/\delta)}{d},
\end{equation}
for all pairs $i\neq j$, where $C>0$ is a constant.

We now turn to determining the relation between the number of hyperplanes and the error probability. Consider a single hyperplane $w$.
For fixed $x_i,x_j$ with angle $\theta_{ij}$ and $w$ uniform on the sphere, the probability that the hyperplane $w$ does not separate the points is,
\[
p = \Pr\big[\sign(w^\top x_i)=\sign(w^\top x_j)\big]
= 1-\frac{\theta_{ij}}{\pi}.
\]
As we have $N$ independent hyperplanes, the probability that none of them separates a specific pair $i\neq j$ is $p^N$, and taking a union bound over the pairs shows that,
\begin{equation} \label{eq:worderror_rand_hyp}
\Pr(E)
\;\le\; \binom{K}{2}\,p^N
\;\le\; \frac{K^2}{2}\,p^N.
\end{equation}
From~(\ref{dist}) we get that with probability at least $1-\delta/2$,
\begin{equation}\label{eq:p}
p 
\le
\frac12
+
C
\sqrt{
\frac{\log(K^2/\delta)}{d}},
\end{equation}
so, in order to ensure that $\Pr(E) \le \delta$, it suffices to choose $N$ such that
$K^2p^N/2 \;\le\; \delta/2$, i.e.,
\begin{equation}\label{eq: achhiev}
N \;\ge\; \frac{\log(K^2/\delta)}{\log(1/p)}
= \frac{2\log K + \log(1/\delta)}{\log(1/p)}.
\end{equation}
Substituting for $p$ from \eqref{eq:p} and doing a Taylor expansion in the denominator yields the claim of the theorem.
\end{proof}

\noindent \emph{Remark.} In the intermediate setting where
$
d=\Theta(\log K),
$
the extreme-angle asymptotics of~\cite{cai2013distributions},
suggest that the minimum angle is strictly positive, and hence the worst-pair collision probability $p$ is strictly smaller than one.  Then, (\ref{eq: achhiev}) suggests that 
$
N=O(\log K)
$
random hyperplanes suffice.
In even lower dimension, when $d$ is fixed,  the results of~\cite{cai2013distributions} imply that the minimum angle scales as,
$
\theta_{\min}
=
\Theta_p\left(K^{-2/(d-1)}\right),
$
which means that a polynomial number of hyperplanes is required.

\begin{proof}[Proof of Theorem~\ref{thm:hyp_necc}] 
Define the codeword $C_k$ corresponding to class $k$, and associated class center $x_k$, as in~\eqref{eq:code_hyper}. 
Fix the hyperplanes $W=(w_1,\dots,w_N)$. Since the class centers $x_i$ are i.i.d. and independent of $W$, the codewords $C_i$ are i.i.d. draws from some distribution $P_W$ on the finite set $\mathcal S=\{\pm1\}^N$. The probability of all codewords being distinct is maximised, and the error probability minimised, when $P_W$ is uniform on $\mathcal S$. The error probability can be calculated in this case using a simple birthday problem argument. Thus, we have,
\[
\Pr(E) \;\ge\;  1-\prod_{j=0}^{K-1}\Bigl(\frac{2^N - j}{2^N}\Bigr) = 1 - \prod_{j=0}^{K-1}\Bigl(1-\frac{j}{2^N}\Bigr).
\]
Using the inequality $1-x \le e^{-x}$, we get,
\[
\prod_{j=0}^{K-1}\Bigl(1-\frac{j}{2^N}\Bigr)
\;\le\;\exp\!\Big(-\sum_{j=0}^{K-1}\frac{j}{2^N}\Big)
=\exp\!\Big(-\tfrac{K(K-1)}{2^{N+1}}\Big),
\]
and so,
\[
\Pr(E) \;\ge\; 1-\exp\!\Big(-\tfrac{K(K-1)}{2^{N+1}}\Big).
\]
Hence, to ensure $\Pr(E)\le \delta$, we need,
\[
1-\exp\!\Big(-\tfrac{K(K-1)}{2^{N+1}}\Big)\;\le\;\delta,
\]
or,
\[
2^{N+1} \;\ge\; \frac{K(K-1)}{-\log(1-\delta)}.
\]
Taking logarithms and using a Taylor expansion, we get,
\[
N \;\ge\; 2\log_2 K + \log_2(1/\delta) - C,
\]
as required, where $C$ is an explicit constant (for which it is possible to show that $C< 3.5)$.
\end{proof}

Next, we turn to the Random Subsets Algorithm. In order to analyse its performance, we will need the following result about linear separability of point sets.

\begin{lemma} 
\label{lem:linear_sep}
Let $x_1,\dots,x_K \stackrel{\mathrm{iid}}{\sim} \normal(0,I_d)$, and let $y_1,\dots,y_K \sim \mathrm{Bern}(1/2)$ be independent random labels that partition the $x_i$'s into two random subsets. The probability that the subsets are linearly separable, or equivalently, that the labeling is realisable by an affine hyperplane in $\R^d$ is,  
\begin{equation*} \label{pdk} 
p(d,K) = \frac{1}{2^{(K-1)}} \sum_{j=0}^{d}\binom{K-1}{j} = \Pr \Big[\,\mathrm{Bin}(K-1,\tfrac12) \le d\,\Big].
\end{equation*}
\end{lemma}

\begin{proof}
First consider the case where the hyperplanes $w$ are homogeneous, i.e., passing through the origin. A labeling $y$ is realizable iff there exists $w\in \R^d$ with,
\[
y_i \, w^\top x_i > 0 \quad \forall i=1,\dots,K,
\]
i.e., all the points $y_i x_i$ lie in some open halfspace in $\R^d$.
As before, the directions
$u_i = {x_i}/{\|x_i\|}$
are i.i.d. uniform on the sphere $S^{d-1}$. Multiplying by $y_i\in\{\pm1\}$ just flips signs, so $y_i u_i$ has the same distribution as $u_i$.
Thus separability of the subsets reduces to the event that $K$ i.i.d. uniform points $\{u_1,\dots,u_K\}$ on $S^{d-1}$ all lie in some open hemisphere.
By Wendel's theorem~\cite{wendel1962problem}, this probability is given by,
\[
p(d,K) = 2^{-(K-1)}\sum_{j=0}^{d-1}\binom{K-1}{j}.
\]
Allowing affine hyperplanes corresponds to replacing $d-1$ with $d$, by a lifting argument. 
\end{proof}

\begin{corollary} \label{cor:separability}
As $K\to \infty$, the probability of separability, $p(d,K)$, exhibits the following three regimes depending on how quickly the dimension $d$ grows compared to $K$.
\begin{enumerate}
\item \textbf{Low-dimensional regime}  
If $d \le (1-\varepsilon)\tfrac{K-1}{2}$ for some fixed $\varepsilon>0$, then, by a lower-tail Chernoff bound on the binomial distribution,
\[
p(d,K) \;\le\; \exp\!\left(-\tfrac{\varepsilon^2}{4}(K-1) \right).
\]
Thus, the probability that a random subset is linearly separable is exponentially small.

\medskip 

\item \textbf{High-dimensional regime.}  
If $d \ge (1+\varepsilon)\tfrac{K-1}{2}$ for some fixed $\varepsilon>0$, then by Chernoff's upper-tail bound for the binomial distribution,
\[
p(d,K) \;\ge\; 1-\exp\!\left(-\frac{\varepsilon^2}{4+2\varepsilon}\,(K -1) \right).
\]
Hence $p(d,K)\to 1$ exponentially fast, and almost every random subset is linearly separable.
\end{enumerate}
\end{corollary}

We shall need the following lemma.
\begin{lemma}
    \label{lem:prob_diff_cond}
    Let $A$ and $B$ be two events in the same probability space. Then, $|\Pr(B|A)-\Pr(B)| \leq 2\Pr(A^c)$.
\end{lemma}

\noindent \emph{Proof.} Its proof follows from a straightforward calculation:
    \begin{align*}
        &|\Pr(B|A)-\Pr(B)| \\
        &= |\Pr(B|A)-\Pr(B|A)\Pr(A)-\Pr(B|A^c)\Pr(A^c)| \\ 
        &\leq |\Pr(B|A)\Pr(A^c)| + |\Pr(B|A^c)\Pr(A^c)| \leq 2\Pr(A^c).
    \end{align*}

\begin{proof}[Proof of Theorem~\ref{thm:subsets_suff}] 
First, consider the low-dimensional regime in which $d \le (1-\varepsilon)\tfrac{K-1}{2}$. Here, the probability that a random labeling is linearly separable is exponentially small. Thus, the expected number of random labellings required to get even one separable subset is,
$$
\frac{1}{p(d,K)} \ge \exp\!\left(\tfrac{\varepsilon^2}{4}(K-1)\right),
$$
which grows exponentially in $K$. So, the random subset scheme is  infeasible in this case.

Henceforth, we consider the high-dimensional regime. Fix $\varepsilon>0$ and suppose that
$
d\ge (1+\varepsilon)\frac{K-1}{2}.
$
Let $y_1,\dots,y_K$ be i.i.d. $\mathrm{Bern}(1/2)$ random labels, and define $A(y)$ to be the event that the subset partitioning induced by the labels $y=(y_1,\dots,y_K)$ is linearly separable. Then, by Corollary~\ref{cor:separability},
\[
\Pr(A(y)^c) \le \rho_K := \exp\Bigl(
-\frac{\varepsilon^2}{4+2\varepsilon} (K-1)
\Bigr).
\]
Now fix a pair of indices $i,j$ and define $B_{ij}$ to be the event that a random $\mathrm{Bern}(1/2)$ labelling $y$ separates the points $x_i$ and $x_j$, i.e., that $y_i\neq y_j$. Clearly, $\Pr(B_{ij})=1/2.$
Hence, it follows from Lemma~\ref{lem:prob_diff_cond} that,
\begin{equation} \label{eq:bernsplit_prob_lbd}
1-\Pr(B_{ij}\mid A(y)) \le \frac12+2 \rho_K.
\end{equation}

The Random Subsets Algorithm generates $N$ independent labellings, each of which partitions the class centers into two subsets. Denote by $N_K$ the number of these labellings that yield linearly separable subsets; and are finally retained. It then follows from~\eqref{eq:bernsplit_prob_lbd} that, that for a fixed pair of indices $i\neq j$, the probability that it is not separated by any of the resulting $N_K$ hyperplanes is,
$$
p \le \left(\frac12+2 \rho_K\right)^{N_K}.
$$
Taking a union bound over all pairs gives,
\begin{equation} \label{eq:ub_subset}
\Pr(E)
\le
\binom K2
\left(\frac12+\rho_K\right)^{N_K},
\end{equation}
and since $\binom K2\le K^2$, we obtain  that in order to ensure that $\Pr(E) \leq \delta$, it suffices if,
\[
N_K
\ge
\frac{
2\ln K+\ln(1/\delta)
}{
-\ln\left(\frac12+ 2 \rho_K\right)
}
=
\frac{
2\log_2 K+\log_2(1/\delta)
}{
1-\log_2(1+4 \rho_K)
}.
\]
Again by doing Taylor expansion in the denominator, we can rewrite the above sufficient condition  as,
\begin{equation} \label{eq:nk}
N_K
\ge
\left( 2\log_2 K+\log_2(1/\delta) \right)
\left( 1+O(\rho_K) \right).
\end{equation}
Next, each generated subset pair is rejected (if it is not linearly separable) with probability at most
$\rho_K$. Hence, if we generate $N=N_K$ subset pairs, then by a union bound
the probability that at least one of them is rejected is at most
$
N_K\rho_K.
$
Combining this with~(\ref{eq:nk}), shows that if we take,
\[
N
\ge
\left(
2\log_2 K+\log_2(1/\delta)
\right)
\left(1+O(\rho_K)\right),
\]
the total error probability is $\Pr(E) \leq 
\delta + O(\rho_K \log K ).
$
\end{proof}

\begin{proof}[Proof of Theorem~\ref{thm:opt_number_hyp}]
Since $N$ hyperplanes create at most $2^N$ distinct sign vectors, the elementary necessary condition $2^N\ge K$ gives the trivial lower bound,
\[
\Nopt \ge  \lceil \log_2 K \rceil ,
\]
which is valid in all dimensions.
The proof for the lower-dimension case, $d< \tfrac{1}{2}\log_2K$, uses the fact that the maximum number of regions determined by $N$ affine hyperplanes in $\mathbb R^d$ is given by,
\begin{equation} \label{eq:region_count}
R(N,d) = \sum_{j=0}^{d}\binom{N}{j} =  2^N\,\Pr\!\big[\mathrm{Bin}(N,\tfrac12)\le d\big],
\end{equation}
which follows from  classical region-counting formulas~\cite{zaslavsky1975facing,matousek2002lectures}. 
Now, a necessary condition to separate $K$ points with $N$ hyperplanes is that $R(N,d) \ge K$. Suppose $N<2d$. Then $N<\log_2 K$ since $d<\tfrac{1}{2}\log_2K$, and from~\eqref{eq:region_count} we observe~that, 
$$
R(N,d) \le \sum_{j=0}^{N} \binom{N}{j} = 2^N < K,
$$
so by contradiction, fewer than $2d$ hyperplanes do not suffice.

Next, if $N\ge 2d$, then $\binom{N}{j}$ is increasing in $j$ for $j\in \{ 0,\dots,d \}$. Hence, using standard binomial coefficient bounds, we obtain from~\eqref{eq:region_count} that, 
$$
R(N,d) = \sum_{j=0}^{d} \binom{N}{j} \le (d+1)\binom{N}{d} < (d+1) \Bigl( \frac{eN}{d} \Bigr)^d.
$$
It follows that a necessary condition in order to ensure that $R(N,d) \ge K$ is,
\begin{align*}
 N \ge  \frac{1}{(d+1)^{1/d}}  \frac{d}{e}  K^{1/d}\geq \frac{1}{2e}dK^{1/d},
\end{align*}
where we have used the fact that $1/{(d+1)^{1/d}}$ is increasing in $d$, and is therefore minimised at $d=1$.
\end{proof}

\subsection{Simulation results}
In previous sections, we introduced the random hyperplanes and random subsets algorithms, described their behavior in different regimes, and derived asymptotic bounds for their performance.
We now simulate the setup to validate how closely our asymptotic predictions match finite-system behavior, examine the dimensional requirements of each algorithm, and observe the existence of distinct regimes. 
We evaluate the performance of the algorithms across various regimes, comparing the number of hyperplanes required to achieve the same error probability, and discuss the trade-off between dimensional requirements and the number of hyperplanes.

We perform the experiments as follows: We sample $K$ classes from the standard Gaussian in~$\mathbb R ^ d$. 
For random hyperplanes, we sample the  $N$ hyperplane normals from the uniform distribution on the sphere. 
For random subsets, we partition the $K$ points into two groups of size approximately $K/2$ by Bernoulli sampling the class labels. We try to find a hyperplane that perfectly separates these subsets, and retain this pair of subsets together with the corresponding hyperplane if this is possible. In each case, we use the hyperplanes to assign an $N$-length codeword to each of the $K$ points as described in Section~\ref{s:algs}, and the experiment is marked as erroneous if any of the codewords are not unique. 
We estimate the probability of error as the empirical ratio of erroneous experiments to the total number of experiments. 
The results are obtained after  $10^5$ repetitions of each experiment.
\subsubsection{Random hyperplanes} 

In Figure~\ref{fig:Pe32LD} (left), we illustrate the qualitative behavior of random hyperplanes in low-dimensional regimes. The number of classes in this case is $K=32$, so  with $d=2,4,6,8$, we have probably not effectively reached the  high-dimensional region $d\gg\log K$. However, as perhaps expected, it is evident than $d=2$ corresponds to a much slower decay compared to the other three cases.  In all cases, it is evident that as the dimension increases, we need fewer hyperplanes to achieve the same probability of error.
This is of course expected from our theory, since increasing the dimension means that we increase the minimum pairwise angle between two points, and therefore reduce the collision probability between any two codewords. 

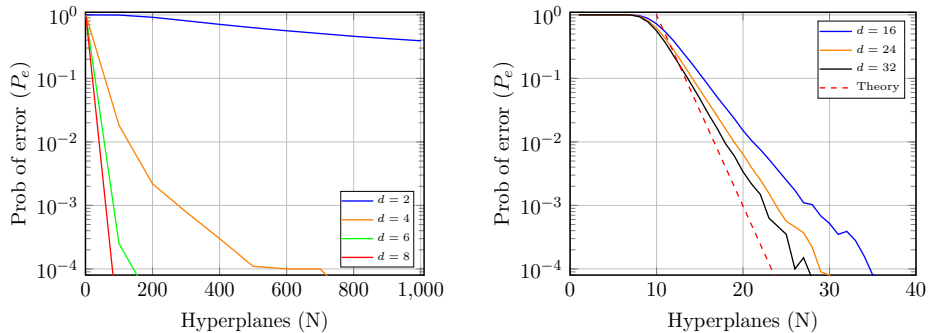
\begin{figure}[ht]
\centering
\resizebox{6cm}{4.5cm}{\begin{tikzpicture}
\begin{axis}[
thick,
xlabel={Hyperplanes (N)},
ylabel={Prob of error $(P_e)$},
xmin=0, xmax=1010,
xtick={0,200,...,1000},
ymin=0.8e-4, ymax=1.1,
ymode=log,
xmajorgrids,
ymajorgrids,
width=0.5\textwidth,
height=0.4\textwidth,
legend style={
font=\tiny,
inner sep=0.1pt,
outer sep=0.1pt,
line width=0.1mm,
legend cell align=left,
},
legend pos=south east
]
\addplot[color=blue, line width=0.7pt] 
coordinates {
(1, 1.0)
(100, 0.993)
(200, 0.91662)
(300, 0.80724)
(400, 0.70889)
(500, 0.62773)
(600, 0.56203)
(700, 0.50847)
(800, 0.45964)
(900, 0.42282)
(1000, 0.39074)
};
\addlegendentry{$d=2$}

\addplot[color=orange, line width=0.7pt] 
coordinates{
(1, 1.0)
(100, 0.01831)
(200, 0.00219)
(300, 0.00079)
(400, 0.0003)
(500, 0.00011)
(600, 0.0001)
(700, 0.0001)
(800, 3e-05)
(900, 1e-05)
(1000, 2e-05)
};
\addlegendentry{$d=4$}

\addplot[color=green, line width=0.7pt] 
coordinates{
(1, 1.0)
(100, 0.00025)
(200, 3e-05)
(300, 0.0)
(400, 0.0)
(500, 0.0)
(600, 0.0)
(700, 0.0)
(800, 0.0)
(900, 0.0)
(1000, 0.0)
};
\addlegendentry{$d=6$}

\addplot[color=red, line width=0.7pt] 
coordinates{
(1, 1.0)
(100, 1e-05)
(200, 0.0)
(300, 0.0)
(400, 0.0)
(500, 0.0)
(600, 0.0)
(700, 0.0)
(800, 0.0)
(900, 0.0)
(1000, 0.0)
};
\addlegendentry{$d=8$}
\end{axis}
\end{tikzpicture}}
~
\resizebox{6cm}{4.5cm}{\begin{tikzpicture}
\begin{axis}[
thick,
xlabel={Hyperplanes (N)},
ylabel={Prob of error $(P_e)$},
xmin=0, xmax=40,
xtick={0,10,...,40},
ymin=0.8e-4, ymax=1.1,
ymode=log,
xmajorgrids,
ymajorgrids,
width=0.5\textwidth,
height=0.4\textwidth,
legend style={
font=\tiny,
inner sep=0.5pt,
outer sep=0.5pt,
line width=0.1mm,
legend cell align=left,
},
legend pos=north east
]
\addplot[color=blue, line width=0.7pt] 
coordinates {
(1, 1.0)
(2, 1.0)
(3, 1.0)
(4, 1.0)
(5, 1.0)
(6, 1.0)
(7, 0.99847)
(8, 0.97277)
(9, 0.88306)
(10, 0.72209)
(11, 0.5427)
(12, 0.38692)
(13, 0.26062)
(14, 0.17699)
(15, 0.11739)
(16, 0.07762)
(17, 0.05101)
(18, 0.03437)
(19, 0.02317)
(20, 0.01511)
(21, 0.01043)
(22, 0.00756)
(23, 0.00531)
(24, 0.00364)
(25, 0.0025)
(26, 0.00173)
(27, 0.0011)
(28, 0.00103)
(29, 0.00068)
(30, 0.00052)
(31, 0.00035)
(32, 0.00039)
(33, 0.00028)
(34, 0.00016)
(35, 8e-05)
(36, 7e-05)
(37, 4e-05)
(38, 6e-05)
(39, 7e-05)
(40, 5e-05)
};
\addlegendentry{$d=16$}


\addplot[color=orange, line width=0.7pt] 
coordinates{
(1, 1.0)
(2, 1.0)
(3, 1.0)
(4, 1.0)
(5, 1.0)
(6, 0.99998)
(7, 0.99666)
(8, 0.95165)
(9, 0.81556)
(10, 0.61906)
(11, 0.42508)
(12, 0.27584)
(13, 0.17352)
(14, 0.10732)
(15, 0.06588)
(16, 0.04129)
(17, 0.02561)
(18, 0.01581)
(19, 0.00977)
(20, 0.0064)
(21, 0.00391)
(22, 0.00253)
(23, 0.00155)
(24, 0.0009)
(25, 0.00056)
(26, 0.00046)
(27, 0.00037)
(28, 0.00022)
(29, 9e-05)
(30, 8e-05)
(31, 4e-05)
(32, 5e-05)
(33, 2e-05)
(34, 3e-05)
(35, 0.0)
(36, 2e-05)
(37, 0.0)
(38, 0.0)
(39, 0.0)
(40, 2e-05)
};
\addlegendentry{$d=24$}


\addplot[color=black, line width = 0.7pt]
coordinates {
(1, 1.0)
(2, 1.0)
(3, 1.0)
(4, 1.0)
(5, 1.0)
(6, 0.99998)
(7, 0.99539)
(8, 0.93681)
(9, 0.77543)
(10, 0.56044)
(11, 0.36824)
(12, 0.22478)
(13, 0.13606)
(14, 0.0816)
(15, 0.04809)
(16, 0.02741)
(17, 0.01638)
(18, 0.00927)
(19, 0.00599)
(20, 0.00338)
(21, 0.00216)
(22, 0.00149)
(23, 0.00063)
(24, 0.00047)
(25, 0.00035)
(26, 0.0001)
(27, 0.00015)
(28, 7e-05)
(29, 3e-05)
(30, 1e-05)
(31, 1e-05)
(32, 0.0)
(33, 0.0)
(34, 0.0)
(35, 1e-05)
(36, 0.0)
(37, 0.0)
(38, 0.0)
(39, 0.0)
(40, 0.0)
};
\addlegendentry{$d=32$}
\addplot[dashed, color=red, line width = 0.7pt]
coordinates {
(26.609640474436812, 1e-05) 
(13.525164236982919, 0.08686) 
(10.000014427022544, 0.99999)
};
\addlegendentry{Theory}
\end{axis}
\end{tikzpicture}}
\caption{Random hyperplanes error probability performance. Left: Low dimensions, Right: High dimensions}
\label{fig:Pe32LD}
\end{figure}

In Figure~\ref{fig:Pe32LD} (right), we illustrate the behavior of random hyperplanes in high dimensions. Since again $K=32$, the presented experiments with $d=16,24,32$ suggests that we have now probably reached the  high-dimensional regime.
When $d\gg \log K$, from our theory we expect all pairwise angles to concentrate on $ \pi/2$, and therefore all pairwise collision probabilities to converge to $p=1/2$. In this case, our previous analysis in~(\ref{eq:worderror_rand_hyp}) shows that a simple upper bound on the probability of error is,
\begin{equation} \label{eq:ub_highd}
\Pr(E)
\;\le\; \binom{K}{2}\,\left (\frac{1}{2}\right ) ^N .
\end{equation}
As expected, Figure~\ref{fig:Pe32LD} verifies that as the dimension $d$ increases, the empirical error curves closely match this theoretical bound, which is plotted as the red dashed line labeled `Theory' in the figure.
This precisely corresponds to the fact that as the dimension increases, the big-$O$ correction term in Theorem~\ref{thm:hyp_suff} vanishes.

\subsubsection{Random subsets} 
In Figure~\ref{fig:PeRSLD32} (left), we illustrate the low-dimensional behavior of the random-subset policy. 
With $K=32$ classes, we observe that for any $d <16$, the probability of error essentially remains close to one irrespective of the number of hyperplanes. This verifies our theory since it is coherent with that fact that if $d<(K-1)/2=15.5$, we expect the probability of random subset partitions being linearly separable to be exponentially small.
For $d=16$, it seems that we have effectively surpassed this threshold, so that the probability of error steadily decreases with $N$.

In Figure~\ref{fig:PeRSLD32} (right) we illustrate the high-dimensional behavior of the random subsets policy, by plotting the experimental results for $d=24,28$, which clearly fall in the high-dimensional regime $d>(K-1)/2$. In sufficiently high dimensions,  we expect the behavior of this scheme to be similar to  random hyperplanes, since we expect essentially all randomly selected partitions to be linearly separable -- and hence used. In  particular, from~(\ref{eq:ub_subset}) we again recover the same upper bound as in~(\ref{eq:ub_highd}) for the probability of error. This is verified in Figure~\ref{fig:PeRSLD32}~(right), where the empirical error curves seem to be very close to the high-dimensional upper bound~(\ref{eq:ub_highd}), signifying that again, the corresponding big-$O$ correction term of Theorem~\ref{thm:subsets_suff}  has vanished.

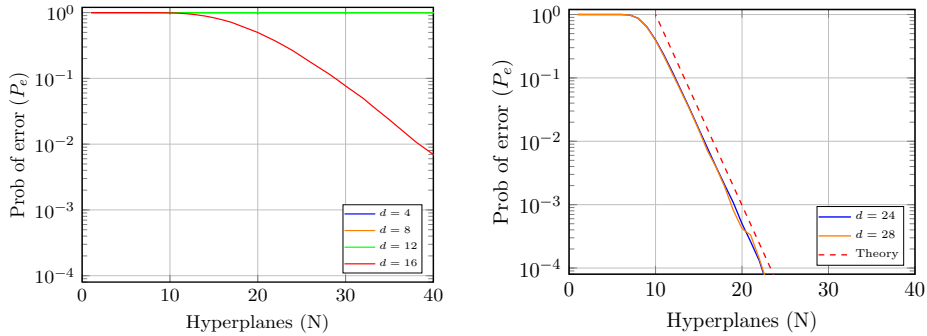
\begin{figure}[ht]
\centering
\resizebox{6cm}{4.5cm}{\begin{tikzpicture}
\begin{axis}[
thick,
xlabel={Hyperplanes (N)},
ylabel={Prob of error $(P_e)$},
xmin=0, xmax=40,
xtick={0,10,...,40},
ymin=0.8e-4, ymax=1.2,
ymode=log,
xmajorgrids,
ymajorgrids,
legend style={
font=\tiny,
inner sep=0.5pt,
outer sep=0.5pt,
line width=0.1mm,
legend cell align=left,
},
legend pos=south east
]
\addplot[color=blue, line width=0.7pt] 
coordinates {
(1.0, 1.0)
(41.0, 1.0)
};
\addlegendentry{$d=4$}

\addplot[color=orange, line width=0.7pt] 
coordinates{
(1.0, 1.0)
(41.0, 1.0)
};
\addlegendentry{$d=8$}

\addplot[color=green, line width=0.7pt] 
coordinates{
(1.0, 1.0)
(27.0, 0.99986)
(41.0, 0.99699)
};
\addlegendentry{$d=12$}

\addplot[color=red, line width=0.7pt] 
coordinates{
(1.0, 1.0)
(7.0, 0.99996)
(9.0, 0.99671)
(10.0, 0.99059)
(11.0, 0.97875)
(12.0, 0.95851)
(13.0, 0.92661)
(14.0, 0.88615)
(15.0, 0.83434)
(16.0, 0.77445)
(17.0, 0.71144)
(20.0, 0.49892)
(22.0, 0.37047)
(23.0, 0.3131)
(24.0, 0.26349)
(26.0, 0.17642)
(28.0, 0.11938)
(29.0, 0.09639)
(30.0, 0.07672)
(32.0, 0.0498)
(33.0, 0.03846)
(35.0, 0.02344)
(38.0, 0.01081)
(41.0, 0.00558)
};
\addlegendentry{$d=16$}
\end{axis}
\end{tikzpicture}}
~
\resizebox{6cm}{4.5cm}{\begin{tikzpicture}
\begin{axis}[
thick,
xlabel={Hyperplanes (N)},
ylabel={Prob of error $(P_e)$},
xmin=0, xmax=40,
xtick={0,10,...,40},
ymin=0.8e-4, ymax=1.2,
ymode=log,
xmajorgrids,
ymajorgrids,
width=0.5\textwidth,
height=0.4\textwidth,
legend style={
font=\tiny,
inner sep=0.5pt,
outer sep=0.5pt,
line width=0.1mm,
legend cell align=left,
},
legend pos=south east
]
\addplot[color=blue, line width=0.7pt] 
coordinates {
(1.0, 1.0)
(5.0, 1.0)
(6.0, 0.99995)
(7.0, 0.98546)
(8.0, 0.87115)
(9.0, 0.639)
(10.0, 0.39822)
(11.0, 0.22562)
(12.0, 0.11974)
(13.0, 0.0614)
(14.0, 0.03154)
(15.0, 0.01567)
(16.0, 0.00783)
(17.0, 0.00395)
(18.0, 0.00204)
(19.0, 0.00107)
(20.0, 0.0005)
(21.0, 0.00026)
(22.0, 0.00013)
(23.0, 5e-05)
(24.0, 1e-05)
(41.0, 0.0)
};
\addlegendentry{$d=24$}

\addplot[color=orange, line width=0.7pt] 
coordinates{
(1.0, 1.0)
(5.0, 1.0)
(6.0, 0.99994)
(7.0, 0.98526)
(8.0, 0.86742)
(9.0, 0.62828)
(10.0, 0.3901)
(11.0, 0.21553)
(12.0, 0.11272)
(13.0, 0.05898)
(14.0, 0.03054)
(15.0, 0.01507)
(16.0, 0.00708)
(17.0, 0.00389)
(18.0, 0.0019)
(19.0, 0.00081)
(20.0, 0.00042)
(21.0, 0.00033)
(22.0, 0.00014)
(24.0, 2e-05)
(25.0, 3e-05)
(27.0, 0.0)
(41.0, 0.0)
};
\addlegendentry{$d=28$}

\addplot[dashed, color=red, line width = 0.7pt]
coordinates {
(26.609640474436812, 1e-05) 
(13.525164236982919, 0.08686) 
(10.000014427022544, 0.99999)
};
\addlegendentry{Theory}

\end{axis}
\end{tikzpicture}}
\caption{Random subsets error probability performance. Left: Low dimensions, Right: High dimensions}
\label{fig:PeRSLD32}
\end{figure}

\vspace*{-0.5  cm}

\subsubsection{Policy comparison}
In Figure~\ref{fig:Comp32}, we compare the performance of the random hyperplane and random subset policies for $K=32$ and $K=1024$ classes, respectively.
From Theorems \ref{thm:hyp_suff} and~\ref{thm:subsets_suff}, we expect the transition to the high-dimensional regime to occur at dimensions that scale as logarithmically with $K$ for random hyperplanes and linearly with $K$ for random subsets. 
Consequently, random hyperplanes begin exhibiting a faster reduction in the probability of error at significantly lower dimensions than random subsets. This fact is illustrated in both Figure~\ref{fig:Comp32}, with the dimensional advantage of random hyperplanes being  substantially more pronounced in Figure~\ref{fig:Comp32}(right), something expected since it corresponds to a much larger number of classes, $K=1024$.
\begin{figure}[ht]
\centering
\resizebox{6cm}{4.5cm}{\begin{tikzpicture}
\begin{axis}[
thick,
xlabel={Hyperplanes (N)},
ylabel={Prob of error $(P_e)$},
xmin=0, xmax=40,
xtick={0,10,...,40},
ymin=0.8e-4, ymax=1.2,
ymode=log,
xmajorgrids,
ymajorgrids,
width=0.5\textwidth,
height=0.4\textwidth,
legend style={
font=\tiny,
inner sep=0.5pt,
outer sep=0.5pt,
line width=0.1mm,
legend cell align=left,
},
legend pos=south west
]
\addplot[color=purple, line width=0.7pt] 
coordinates{
(2.0, 1.0)
(6.0, 0.999980001449585)
(8.0, 0.986010047197342)
(10.0, 0.8115600341558457)
(12.0, 0.5013100227713585)
(14.0, 0.2672300134599208)
(16.0, 0.1360200053453445)
(18.0, 0.0667200032621622)
(20.0, 0.0348600012436509)
(22.0, 0.0194400009699165)
(24.0, 0.0098300004471093)
(26.0, 0.0060500002698972)
(28.0, 0.0031900001177564)
(30.0, 0.0020000000775326)
(32.0, 0.0010800000454764)
(34.0, 0.0005900000268593)
(38.0, 0.0002700000128243)
(40.0,0.0002300000109244138)
(50.0, 0.0)
(80.0, 0.0)
};
\addlegendentry{Hyperplanes $d=12$}

\addplot[color=orange, line width=0.7pt] 
coordinates {
(1.0, 1.0)
(6.0, 1.0)
(7.0, 0.99847)
(8.0, 0.97277)
(9.0, 0.88306)
(10.0, 0.72209)
(11.0, 0.5427)
(12.0, 0.38692)
(13.0, 0.26062)
(14.0, 0.17699)
(15.0, 0.11739)
(16.0, 0.07762)
(17.0, 0.05101)
(18.0, 0.03437)
(19.0, 0.02317)
(20.0, 0.01511)
(21.0, 0.01043)
(22.0, 0.00756)
(23.0, 0.00531)
(24.0, 0.00364)
(25.0, 0.0025)
(27.0, 0.0011)
(28.0, 0.00103)
(29.0, 0.00068)
(31.0, 0.00035)
(32.0, 0.00039)
(35.0, 8e-05)
(40.0, 5e-05)
};
\addlegendentry{Hyperplanes $d=16$}

\addplot[color=blue, line width=0.7pt] 
coordinates {
(1.0, 1.0)
(6.0, 0.99999)
(7.0, 0.99585)
(8.0, 0.94464)
(9.0, 0.79026)
(10.0, 0.58397)
(11.0, 0.39271)
(12.0, 0.24712)
(13.0, 0.15305)
(14.0, 0.09288)
(15.0, 0.0553)
(16.0, 0.03434)
(17.0, 0.02015)
(18.0, 0.01186)
(19.0, 0.00705)
(20.0, 0.00442)
(21.0, 0.0026)
(22.0, 0.00171)
(23.0, 0.001)
(25.0, 0.00038)
(26.0, 0.00019)
(31.0, 2e-05)
(40.0, 0.0)
};
\addlegendentry{Hyperplanes $d=28$}

\addplot[dashed, color=purple, line width=0.7pt] 
coordinates {
(1.0, 1.0)
(22.0, 1.0)
(26.0, 0.99977)
(27.0, 0.99986)
(30.0, 0.99952)
(32.0, 0.99947)
(34.0, 0.99916)
(36.0, 0.99861)
(37.0, 0.99858)
(39.0, 0.99732)
(41.0, 0.99699)
};
\addlegendentry{Subsets $d=12$}

\addplot[dashed, color=orange, line width=0.7pt] 
coordinates{
(1.0, 1.0)
(7.0, 0.99996)
(8.0, 0.99905)
(9.0, 0.99671)
(10.0, 0.99059)
(11.0, 0.97875)
(12.0, 0.95851)
(13.0, 0.92661)
(14.0, 0.88615)
(15.0, 0.83434)
(16.0, 0.77445)
(17.0, 0.71144)
(18.0, 0.64137)
(19.0, 0.56975)
(20.0, 0.49892)
(21.0, 0.4348)
(22.0, 0.37047)
(23.0, 0.3131)
(24.0, 0.26349)
(25.0, 0.21944)
(26.0, 0.17642)
(27.0, 0.14726)
(28.0, 0.11938)
(29.0, 0.09639)
(30.0, 0.07672)
(31.0, 0.06234)
(32.0, 0.0498)
(33.0, 0.03846)
(34.0, 0.0305)
(35.0, 0.02344)
(37.0, 0.0146)
(38.0, 0.01081)
(40.0, 0.0065)
(41.0, 0.00558)
};
\addlegendentry{Subsets $d=16$}

\addplot[dashed, color=blue, line width=0.7pt] 
coordinates{
(1.0, 1.0)
(6.0, 0.99994)
(7.0, 0.98526)
(8.0, 0.86742)
(9.0, 0.62828)
(10.0, 0.3901)
(11.0, 0.21553)
(12.0, 0.11272)
(13.0, 0.05898)
(14.0, 0.03054)
(15.0, 0.01507)
(16.0, 0.00708)
(17.0, 0.00389)
(18.0, 0.0019)
(19.0, 0.00081)
(20.0, 0.00042)
(22.0, 0.00014)
(24.0, 2e-05)
(41.0, 0.0)
};
\addlegendentry{Subsets $d=28$}

\addplot[dotted, color=red, line width = 0.7pt]
coordinates {
(26.609640474436812, 1e-05) 
(13.525164236982919, 0.08686) 
(10.000014427022544, 0.99999)
};
\addlegendentry{Theory}

\end{axis}
\end{tikzpicture}}
~
\resizebox{6cm}{4.5cm}{\begin{tikzpicture}
\begin{axis}[
thick,
xlabel={Hyperplanes (N)},
ylabel={Prob of error $(P_e)$},
xmin=0, xmax=40,
xtick={0,10,...,40},
ymin=0.8e-4, ymax=1.2,
ymode=log,
xmajorgrids,
ymajorgrids,
width=0.5\textwidth,
height=0.4\textwidth,
legend style={
font=\tiny,
inner sep=0.5pt,
outer sep=0.5pt,
line width=0.1mm,
legend cell align=left,
},
legend pos=south west
]

\addplot[color=purple, line width=0.7pt]
coordinates {
(1.0, 1.0)
(15.0, 1.0)
(16.0, 0.99986)
(17.0, 0.98795)
(18.0, 0.89672)
(19.0, 0.68072)
(20.0, 0.4375)
(21.0, 0.25462)
(22.0, 0.1387)
(23.0, 0.07315)
(24.0, 0.03825)
(25.0, 0.0192)
(26.0, 0.0095)
(27.0, 0.00457)
(28.0, 0.00271)
(29.0, 0.0013)
(30.0, 0.00058)
(32.0, 0.00015)
(40.0, 0.0)
};
\addlegendentry{Hyperplanes $d=512$}

\addplot[color=orange, line width=0.7pt]
coordinates {
(1.0, 1.0)
(15.0, 1.0)
(16.0, 0.99983)
(17.0, 0.98841)
(18.0, 0.89451)
(19.0, 0.68121)
(20.0, 0.43709)
(21.0, 0.25597)
(22.0, 0.13967)
(23.0, 0.07348)
(24.0, 0.03702)
(25.0, 0.01958)
(26.0, 0.00986)
(27.0, 0.0049)
(28.0, 0.00251)
(29.0, 0.00123)
(31.0, 0.0003)
(33.0, 3e-05)
(40.0, 0.0)
};
\addlegendentry{Hyperplanes $d=528$}

\addplot[color= blue, line width=0.7pt]
coordinates {
(1.0, 1.0)
(15.0, 1.0)
(16.0, 0.99979)
(17.0, 0.98758)
(18.0, 0.89042)
(19.0, 0.67232)
(20.0, 0.43208)
(21.0, 0.24909)
(22.0, 0.1356)
(23.0, 0.0724)
(24.0, 0.0381)
(25.0, 0.01889)
(26.0, 0.00962)
(27.0, 0.00457)
(28.0, 0.00239)
(29.0, 0.00129)
(30.0, 0.00093)
(31.0, 0.00038)
(32.0, 0.00012)
(33.0, 0.00011)
(34.0, 5e-05)
(35.0, 2e-05)
(36.0, 3e-05)
(38.0, 0.0)
(40.0, 0.0)
};
\addlegendentry{Hyperplanes $d=608$}

\addplot[dashed, color=purple, line width=0.7pt]
coordinates {
(1.0, 1.0)
(41.0, 1.0)
};
\addlegendentry{Subsets $d=512$}

\addplot[dashed, color=orange, line width=0.7pt]
coordinates {
(1.0, 1.0)
(27.0, 0.99995)
(29.0, 0.99983)
(32.0, 0.99888)
(33.0, 0.99826)
(35.0, 0.99575)
(36.0, 0.99398)
(37.0, 0.99129)
(38.0, 0.98767)
(39.0, 0.98434)
(40.0, 0.97838)
(41.0, 0.97102)
};
\addlegendentry{Subsets $d=528$}

\addplot[dashed, color=blue, line width=0.7pt] 
coordinates{
(1.0, 1.0)
(15.0, 1.0)
(16.0, 0.99974)
(17.0, 0.98279)
(18.0, 0.87131)
(19.0, 0.64613)
(20.0, 0.40962)
(21.0, 0.23529)
(22.0, 0.12388)
(23.0, 0.06572)
(24.0, 0.03305)
(25.0, 0.01773)
(26.0, 0.00887)
(27.0, 0.00454)
(28.0, 0.00211)
(29.0, 0.00096)
(30.0, 0.00055)
(31.0, 0.00028)
(32.0, 0.00023)
(33.0, 4e-05)
(34.0, 6e-05)
(35.0, 0.0)
(41.0, 0.0)
};
\addlegendentry{$d=608$}

\addplot[dotted, color=red, line width = 0.7pt]
coordinates {
(36.609640474436816, 1e-05)
(23.525164236982917, 0.08686) 
(20.000014427022546, 0.99999)
};
\addlegendentry{Theory}

\end{axis}
\end{tikzpicture}}
\caption{Policy comparison. Left: $K=32$ classes, Right:  $K=1024$ classes}
\label{fig:Comp32}
\end{figure}
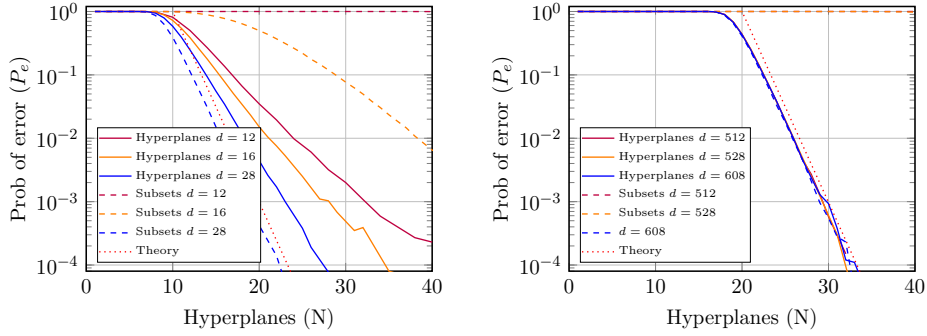

 \textbf{Summary.} Overall, the experimental results validate the theoretical scaling laws derived in the previous sections and highlight the fundamental difference between the two constructions. 
Although both policies require a comparable number of hyperplanes in their respective high-dimensional regimes, the random hyperplane policy attains the same reliability while requiring significantly fewer dimensions. 
As the number of classes increases, this dimensional advantage becomes increasingly pronounced, making the random hyperplane policy substantially more scalable than the random subset policy. For this reason, for the remaining of this paper where we focus on the full noisy-observation setting, we restrict our attention to the random hyperplanes scheme, and provide fundamental limits for its performance in the presence of noise.

\section{Orthogonal hyperplanes with noise} \label{s:orth_hyps}

As an intermediate step in understanding the fundamental limits of the random hyperplanes scheme in the presence of noise, we first consider a slightly simpler (in terms of analysis) scheme, where the hyperplanes are chosen orthogonal to each other. This simplification allows to get precise theoretical results, and at the same time gives a strong indication on what we expect to happen in sufficiently high dimensions in the original setting, where randomly chosen hyperplanes will tend to be orthogonal to each other; again, this is because of the concentration of random angles to $\pi/2$ in sufficiently high dimensions.

\subsection{Model description} \label{s:orth_descr}

Let \(u_1,\dots,u_N\in \mathbb{R}^d\) be orthonormal, with \(N\le d\). Let the $K$ class centers be,
\[
x_1,\dots,x_K \overset{iid}{\sim} \mathcal{N}(0,\tau^2 I_d),
\]
and suppose the true class is \(a\in\{1,\dots,K\}\). Define the noisy observation,
\[
y=x_a+z,
\qquad
z \sim \mathcal{N}(0,\sigma^2 I_d),
\]
independently. For each class \(k\), define its binary codeword,
\[
C_k=(C_{k,1},\dots,C_{k,N})\in\{\pm 1\}^N,
\qquad
C_{k,i}:=\operatorname{sign}(u_i^\top x_k),
\]
and define the observed binary vector,
\[
B=(B_1,\dots,B_N)\in\{\pm 1\}^N,
\qquad
B_i:=\operatorname{sign}(u_i^\top y).
\]
For the true class \(a\), let,
\[
p_{a,i}:=\Pr(B_i\neq C_{a,i}\mid x_a,u_1,\dots,u_N)
=
Q\!\left(\frac{|u_i^\top x_a|}{\sigma}\right),
\]
where \(Q(t)=\Pr(Z\ge t)\) for \(Z\sim \mathcal{N}(0,1)\), and let $\rho:={\tau}/{\sigma}$ denote the signal to noise ratio~(SNR).

\medskip

\textbf{Communications formulation.} This classification problem reduces very naturally to a communication problem. In particular, since the centers are sampled i.i.d., and the hyperplanes are orthogonal, we have $K$ codewords $C_1,...,C_K$ of length $N$ that are sampled uniformly and independently from $\{\pm1\} ^N$. Since the hyperplanes are orthogonal, the transmitted codeword $C_a$ is passed through a channel in which each of the $N$ bits (coordinates) $i$ is flipped independently with probability $p_{a,i}$ as defined above, which is geometry-dependent. 

Before considering the original geometry-dependent problem above, suppose we first simplify the analysis by removing the geometry and 
homogenizing the model. In particular, suppose we  replace the coordinate-dependent flip probabilities $p_{a,i}$
by a single flip probability~$p$, defined as the expected value of $p_{a,i}$. With this simplification, the problem reduces to communication over a Binary Symmetry Channel (BSC), with crossover probability~$p$.

\subsection{Further reduction to BSC}

\begin{lemma}
Let \(x_a\sim \mathcal N(0,\tau^2 I_d)\),  \(z\sim \mathcal N(0,\sigma^2 I_d)\) independent of $x_a$,  \(u_i\in S^{d-1}\) be a fixed unit vector, and $\rho = \tau / \sigma$. Define,
\[
C_{a,i}=\operatorname{sign}(u_i^\top x_a),
\qquad
B_i=\operatorname{sign}(u_i^\top(x_a+z)).
\]
The conditional probability that the \(i\)-th bit of the true codeword $C_a$ is flipped is,
\[
p_{a,i}
:=
\Pr\{B_i\neq C_{a,i}\mid x_a,u_i\}
=
Q\left(\frac{|u_i^\top x_a|}{\sigma}\right),
\]
and its expected value is given by,
\begin{equation} \label{eq:cross_prob}
p:=\mathbb E_{x_a}[p_{a,i}]
=
\frac1\pi \arctan\frac1\rho.
\end{equation}
\end{lemma}

\noindent \textit {Proof.}
Condition on \(x_a\) and  \(u_i\). Since
\(z\sim \mathcal N(0,\sigma^2 I_d)\), we have
$
u_i^\top z\sim \mathcal N(0,\sigma^2).
$
The \(i\)-th bit is flipped when the noise changes the sign of
\(u_i^\top x_a\), i.e.,
$
\operatorname{sign}(u_i^\top x_a+u_i^\top z)
\neq
\operatorname{sign}(u_i^\top x_a).
$
Therefore,
\[
p_{a,i}
=
\Pr\{B_i\neq C_{a,i}\mid x_a,u_i\}
=
Q\left(\frac{|u_i^\top x_a|}{\sigma}\right).
\]
Now averaging over \(x_a\) gives,
\[
\mathbb E[p_{a,i}]
=
\mathbb E_G[Q(\rho |G|)],
\]
where $G\sim\mathcal N(0,1)$.
Finally, in order to compute this expectation, let \(W\sim\mathcal N(0,1)\) be independent of \(G\). Then by the definition of the Q-function,
\[
\mathbb E_G[Q(\rho |G|)]
=
\Pr\{W>\rho |G|\}.
\]
The vector \((G,W)\) is rotationally invariant in \(\mathbb R^2\), and the event
\(\{W>\rho |G|\}\) is a cone with opening angle \(2\arctan(1/\rho)\). Hence,
\[ p =
\mathbb E[p_{a,i}] = 
\Pr\{W>\rho |G|\}
=
\frac{2\arctan(1/\rho)}{2\pi}
=
\frac1\pi\arctan\frac1\rho.
\qed
\]

Using this appropriate crossover probability $p$, our setting is now reduced to a standard \(\mathrm{BSC}(p)\). Standard analysis shows that Hamming decoding is maximum likelihood/optimal in this case, and similarly gives capacity and error exponents results, for example as below. Throughout this section, $\Pr (E)$ denotes the unconditional probability of error,
where the probability is taken over the random class centers
\(x_1,\dots,x_K\), the hyperplanes \(u_1,\dots,u_N\), the noise~\(z\), and the
choice of the transmitted class \(a\), i.e.,
$
\Pr (E)
:=
\Pr(\widehat a\neq a).
$

\paragraph{Capacity.}
A necessary and sufficient condition for reliable recovery asymptotically in the
regime \(N\le d\), \(N\to\infty\), under the BSC assumptions, is given by the channel coding theorem,
\[
\frac{\log_2 K}{N} < 1-h_2(p),
\qquad
p=\frac1\pi\arctan\frac1\rho.
\]

\begin{theorem} \label{thm:bsc} 
Under the BSC assumptions, a sufficient condition to ensure that the probability of error satisfies \(\Pr (E)\le \delta\) is,
\[
N\ge \frac{\log((K-1)/\delta)}{E_{\mathrm{BSC}}(\rho)},
\]
where, $\rho = \tau / \sigma $, and  $E_{\mathrm{BSC}}(\rho)$ is given by,
\begin{equation} \label{eq:Ebsc}
E_{\mathrm{BSC}}(\rho):=
-\log\!\left(\frac12+\sqrt{p(1-p)}\right),
\end{equation}
with $p = \tfrac{1}{\pi} \arctan \tfrac{1}{\rho}$, as before in (\ref{eq:cross_prob}).
\end{theorem}

\begin{proof}
The proof follows from Chernoff analysis applied to the homogenized BSC setting. It is omitted here since it can be seen as a special case of Theorem~\ref{thm:orth_ham} that considers the more general setting with geometry-dependent bitflip probabilities $p_{a,i}$, where we can set all  $p_{a,i}=p$.
\end{proof}

\subsection{Hamming decoding}

Now, we return to our original geometry-dependent channel with independent bitflips with probability $p_{a,i}$. As a first step, we consider a Hamming decoder,
\[
\widehat a\in\arg\min_{m\in\{1,\dots,K\}}d_H(B,C_m),
\]
where $d_H(B,C_m)$ denotes the Hamming distance between the binary vectors $B,C_m$.
Performing a Chernoff-bound analysis to provide an achievability result for the sufficient number of hyperplanes, shows that our previous bound still works here, as expected.

\begin{theorem}
\label{thm:orth_ham}
In the orthogonal hyperplanes setting of Section~\ref{s:orth_descr}, a sufficient condition to ensure that the probability of error satisfies \(\Pr (E)\le \delta\)~is,
\[
N\ge \frac{\log((K-1)/\delta)}{E_{\mathrm{BSC}}(\rho)},
\] 
where $E_{\mathrm{BSC} } (\rho)$ is given by (\ref{eq:Ebsc}).
\end{theorem} 
\begin{proof} 
Fix a competitor \(b\neq a\), and define,
$
D_{ab}:=\{i:C_{a,i}\neq C_{b,i}\}.
$
For \(i\notin D_{ab}\), the two codewords agree, so coordinate \(i\) contributes equally to
\(d_H(B,C_a)\) and \(d_H(B,C_b)\); thus only coordinates in \(D_{ab}\) matter.
For \(i\in D_{ab}\), define,
\[
M_i^{(ab)}
:=
\mathbf 1\{B_i\neq C_{b,i}\}
-
\mathbf 1\{B_i\neq C_{a,i}\},
\]
and since for \(i\in D_{ab}\) we have $C_{a,i}\neq C_{b,i}$, we get,
\[
\Pr(M_i^{(ab)}=+1 |x_a, x_b, u_i)=1-p_{a,i},
\qquad
\Pr(M_i^{(ab)}=-1 | x_a, x_b, u_i)=p_{a,i}.
\]
Pairwise confusion occurs when,
\[
d_H(B,C_b)-d_H(B,C_a) = \sum_{i\in D_{ab}}M_i^{(ab)}\le0.
\]
Let $x_1^K$ and $u_1 ^N$ denote the vector concatenations of $x_1,...,x_K$ and $u_1,..., u_N$. Fix a global Chernoff parameter \(s\ge0\). By Markov's inequality,
\[
\Pr\left(
\sum_{i\in D_{ab}}M_i^{(ab)}\le 0
\,\middle|\,
x_1^K, u_1 ^N
\right)
\le
\mathbb E\left[
\exp\left(
-s\sum_{i\in D_{ab}}M_i^{(ab)}
\right)
\,\middle|\,
x_1^K, u_1 ^N
\right].
\]
Conditional on \(x_a\) and $u_1 ^N$, the projected noises
$
W_i:=u_i^\top z,\  i=1,\dots,N,
$
are independent \(\mathcal N(0,\sigma^2)\) random variables, since the hyperplanes are orthogonal, and hence
the variables \(M_i^{(ab)}\) are conditionally independent. Therefore,
\[
\mathbb E\left[
\exp\left(
-s\sum_{i\in D_{ab}}M_i^{(ab)}
\right)
\,\middle|\,
x_1^K, u_1 ^N
\right]
=
\prod_{i\in D_{ab}}
\mathbb E\left[
e^{-sM_i^{(ab)}}
\,\middle|\,
x_1^K, u_1 ^N
\right],
\]
and for \(i\in D_{ab}\),
\[
\mathbb E\left[
e^{-sM_i^{(ab)}}
\,\middle|\,
x_1^K, u_1 ^N
\right]
=
(1-p_{a,i})e^{-s}+p_{a,i}e^s.
\]
Therefore, for any \(s\ge0\),
\[
\Pr\left(
d_H(B,C_b)\le d_H(B,C_a)
\,\middle|\,
x_1^K, u_1 ^N
\right)
\le
\prod_{i\in D_{ab}}
\left((1-p_{a,i})e^{-s}+p_{a,i}e^s\right),
\]
and taking a union bound over \(b\neq a\), for every \(s\ge0\),
\[
\Pr\left(
E
\,\middle|\,
x_1^K, u_1 ^N
\right)
\le
\sum_{b\neq a}
\prod_{i\in D_{ab}}
\left((1-p_{a,i})e^{-s}+p_{a,i} e^s\right).
\]
Now fix \(x_a\) and $u_1^N$, and average over one random competitor \(x_b\). Since
$
u_i^\top x_b\overset{\mathrm{iid}}{\sim}\mathcal N(0,\tau^2),
$
the signs \(C_{b,1},\dots,C_{b,N}\) are independent fair signs. Therefore, conditional on \(C_a\),
$
\mathbf 1\{i\in D_{ab}\}
\sim \operatorname{Bernoulli}(1/2),
$
independently over \(i\). Thus,
\[
\mathbb E_{x_b}
\left[
\prod_{i\in D_{ab}}
\left((1-p_{a,i})e^{-s}+p_{a,i} e^s\right)
\,\middle|\,
x_a,u_1^N
\right]
=
\prod_{i=1}^N
\left[
\frac12+
\frac12\left((1-p_{a,i})e^{-s}+p_{a,i} e^s\right)
\right].
\]
Using the notation,
\[
m_s(x):=
\frac12+\frac12\left((1-x)e^{-s}+xe^s\right),
\]
summing over the \(K-1\) competitors, and taking the infimum over $s$ to tighten the bound yields,
\begin{equation}\label{cond_bound}
\Pr\left(
E
\,\middle|\,
x_a, u_1 ^N
\right)
\le
(K-1) \inf_{s\ge0}\prod_{i=1}^N m_s(p_{a,i}).
\end{equation}
Finally, average over the transmitted point \(x_a\). Since \(u_1,\dots,u_N\) are orthonormal and
$
x_a\sim\mathcal N(0,\tau^2I_d),
$
we have
$
u_i^\top x_a\overset{\mathrm{iid}}{\sim}\mathcal N(0,\tau^2),
$
and,
\[
p_{a,i}
=
Q\left(\frac{|u_i^\top x_a|}{\sigma}\right)
\overset{d}{=}
Q(\rho |Z_i|),
\]
where \(Z_1,\dots,Z_N\overset{\mathrm{iid}}{\sim}\mathcal N(0,1)\). 
We can now relax the bound by moving the infimum outside the expectation:
\[
\Pr (E)
\le
(K-1)   \mathbb E
\left[ \inf_{s\ge0}\prod_{i=1}^N m_s(p_{a,i}) \right ] \leq 
(K-1) \inf_{s\ge0} \mathbb E
\left[ \prod_{i=1}^N m_s(p_{a,i}) \right ] 
=
(K-1)\inf_{s\ge0}
\left(m_s(p)\right)^N ,
\]
where the last equality follows since the $p_{a,i}$ are i.i.d. with expectation $\mathbb E[p_{a,i}]=p$ as before, and $
m_s(x)$ 
is linear in $x$.
Optimizing over \(s\ge0\), the minimum value of $m_s(p)$ is,
\[
m_s(p) \ge \frac12+\sqrt{p(1-p)},
\]
and substituting gives,
\begin{equation} \label{eq:orth_ub}
\Pr (E)
\le
(K-1)
\left(
\frac12+\sqrt{p(1-p)}
\right)^N
=
(K-1) \ e^{-N E_{\mathrm{BSC}}(\rho)},
\end{equation}
which gives the required error exponent and completes the proof after rearranging.
\end{proof}

\noindent \emph{Remark.} 
Theorem~\ref{thm:orth_ham} formally shows that the achievabilty condition under the BSC assumptions is still sufficient for the geometry-dependent case. Importantly, it shows that $N = O(\log K )$ hyperplanes are sufficient, with a multiplicative constant depending on the signal to noise ratio. A somewhat tighter constant can be obtained by exploiting the geometry. In specific, we can get a refinement by avoiding to move the expectation inside the infimum in~(\ref{cond_bound}), define,
\[
E_{q,N}(\rho)
:=
-\frac1N
\log
\mathbb E
\left[
\inf_{s\ge0}
\prod_{i=1}^N m_s(p_{a,i})
\right],
\]
and use it in place of $E_{\mathrm{BSC}}(\rho)$. However, this requires a much more expensive optimization, and in practice it was found to give very similar results with $E_{\mathrm{BSC}}(\rho)$ in simulation experiments.

\subsection{Capacity analysis}

For a fixed realization of the geometry, the \(N\) coordinates form independent
binary symmetric channels with crossover probabilities
\(p_{a,1},\dots,p_{a,N}\). If these realized reliabilities are known to the
decoder, the total capacity over the \(N\) bits is,
\[
\sum_{i=1}^N \left(1-h_2(p_{a,i})\right).
\]
Thus the geometry-aware capacity benchmark is,
\[
\log_2 K
<
\sum_{i=1}^N \left(1-h_2(p_{a,i})\right).
\]
Since the \(p_{a,i}\)'s are i.i.d. after averaging over the random geometry, by the
law of large numbers,
\[
\frac1N
\sum_{i=1}^N \left(1-h_2(p_{a,i})\right)
\longrightarrow
\mathbb E\left[1-h_2(p_{a,1})\right],
\]
and the geometry-aware capacity benchmark reads,
\[
\frac{\log_2 K}{N}
<
C_{\mathrm{geom}}(\rho)+o(1),
\]
where the geometry-aware capacity is,
\begin{equation} \label{eq:c_geom}
C_{\mathrm{geom}}(\rho)
:=
1 -\mathbb E\left[h_2(Q(\rho |Z|))\right], \qquad Z\sim\mathcal N(0,1).
\end{equation}

\paragraph{Comparison via Jensen's inequality.}
Since the binary entropy function \(h_2\) is concave on \([0,1]\), Jensen's inequality immediately gives,
\[
C_{\mathrm{geom}}(\rho) = 1-\mathbb E[h_2(Q(\rho |Z|))]
\ge
1- h_2 ( \mathbb E[Q(\rho |Z|)]) = 
1-h_2(p) = C_{\mathrm{BSC}}(\rho).
\]
Thus the geometry-aware capacity is always at least as large as the homogeneous
BSC capacity. The improvement comes from exploiting the realized reliabilities
of the individual bits rather than replacing them by their average crossover
probability.

\subsection{Maximum likelihood decoding}

It is easy to see that Hamming decoding is not optimal in the geometry-aware setting. We now consider the optimal maximum likelihood  decoder, which uses both the binary signs and the
corresponding margins. For each class \(m\) and coordinate \(i\), 
\[
p_{m,i}
:=
Q\left(\frac{|u_i^\top x_m|}{\sigma}\right).
\]
Under hypothesis \(m\), the observed bit \(B_i\) satisfies,
\[
\Pr(B_i=C_{m,i}\mid m)=1-p_{m,i},
\qquad
\Pr(B_i\neq C_{m,i}\mid m)=p_{m,i},
\]
and the likelihood of a candidate class \(m\) is,
\[
P_m(B)
=
\prod_{i=1}^N
\left[
(1-p_{m,i})\mathbf 1\{B_i=C_{m,i}\}
+
p_{m,i}\mathbf 1\{B_i\neq C_{m,i}\}
\right].
\]
The geometry-aware maximum likelihood decoder is,
\[
\widehat m_{\mathrm{ML}}
\in
\arg\max_{m\in\{1,\dots,K\}} P_m(B).
\]

\begin{theorem}
\label{thm:orth_mle}
In the orthogonal hyperplanes setting of Section~\ref{s:orth_descr}, a sufficient condition to ensure that the probability of error satisfies \(\Pr (E)\le \delta\)~is,
\[
N\ge \frac{\log((K-1)/\delta)}{E_\mathrm{ML}(\rho)},
\] 
where, with $Z\sim\mathcal N(0,1)$, the corresponding  ML error exponent is defined as,
\begin{equation} \label{eq:orth_ml}
E_{\mathrm{ML}}(\rho)
=
-\log\left[
\frac12
\left(
\mathbb E_Z\left[
\sqrt{1-Q(\rho|Z|)}
+
\sqrt{Q(\rho|Z|)}
\right]
\right)^2
\right].
\end{equation}
\end{theorem}

\begin{proof}
The proof follows a Chernoff information analysis. Suppose the transmitted class is \(a\). Fix a competitor \(b\neq a\). Let
\(P_a\) and \(P_b\) denote the product distributions of the observed bit vector~\(B\) under hypotheses
\(a\) and \(b\), respectively, given
the geometry. The pairwise ML confusion event is,
$
\{P_b(B)\ge P_a(B)\}.
$
Conditional on $x_1^K, u_1^N$, for any \(s\in[0,1]\), Markov's inequality gives,
\[
\Pr\left(
P_b(B)\ge P_a(B)
\,\middle|\,
x_1^K,u_1^N
\right)
=
\Pr\left[
\left(\frac{P_b(B)}{P_a(B)}\right)^s\ge 1
\,\middle|\,
x_1^K,u_1^N
\right]
\le
\mathbb E\left[
\left(\frac{P_b(B)}{P_a(B)}\right)^s
\,\middle|\,
x_1^K,u_1^N
\right].
\]
Since the bits \(B_1,\dots,B_N\) are independent under each hypothesis (as before),
\[
\mathbb E\left[
\left(\frac{P_b(B)}{P_a(B)}\right)^s
\,\middle|\,
x_1^K,u_1^N
\right]
=
\prod_{i=1}^N
\sum_{\beta\in\{\pm1\}}
P_{a,i}(\beta)
\left(\frac{P_{b,i}(\beta)}{P_{a,i}(\beta)}\right)^s
=
\prod_{i=1}^N
\sum_{\beta\in\{\pm1\}}
P_{a,i}(\beta)^{1-s}
P_{b,i}(\beta)^s.
\]
Defining the coordinate Chernoff coefficient,
\[
Z_i^{(ab)}(s)
:=
\sum_{\beta\in\{\pm1\}}
P_{a,i}(\beta)^{1-s}
P_{b,i}(\beta)^s,
\]
gives the conditional pairwise error bound,
\begin{equation} \label{eq:orth_pair}
\Pr (P_b(B)\ge P_a(B) \mid  x_1^K,u_1^N)
\le
\prod_{i=1}^N Z_i^{(ab)}(s).
\end{equation}
Now average over the random geometry $x_a,x_b$. Since \(x_a\) and \(x_b\) are i.i.d.
Gaussian, as before,
\[
p_{a,i}\overset d=Q(\rho |Z_a|),
\qquad
p_{b,i}\overset d=Q(\rho |Z_b|),
\]
where \(Z_a,Z_b\overset{\mathrm{iid}}{\sim}\mathcal N(0,1)\). Furthermore,
the signs \(C_{a,i}\) and \(C_{b,i}\) are independent fair signs, independent
of the magnitudes:
$
\Pr(C_{a,i}=C_{b,i})=
\Pr(C_{a,i}\neq C_{b,i})=1/2.
$
If \(C_{a,i}=C_{b,i}\), then,
\[
Z_i^{(ab)}(s)
=
(1-p_{a,i})^{1-s}(1-p_{b,i})^s
+
p_{a,i}^{1-s}p_{b,i}^s,
\]
and if \(C_{a,i}\neq C_{b,i}\), then,
\[
Z_i^{(ab)}(s)
=
(1-p_{a,i})^{1-s}p_{b,i}^s
+
p_{a,i}^{1-s}(1-p_{b,i})^s.
\]
Therefore,  denoting for simplicity of notation for the corresponding coordinate $i$, $p_a=p_{a,i}$,
\[
\mathbb E[Z_i^{(ab)}(s)]
=
\frac12
\mathbb E\left[
(1-p_a)^{1-s}(1-p_b)^s
+
p_a^{1-s}p_b^s
\right]
+
\frac12
\mathbb E\left[
(1-p_a)^{1-s}p_b^s
+
p_a^{1-s}(1-p_b)^s
\right].
\]
Combining terms,
\[
\mathbb E[Z_i^{(ab)}(s)]
=
\frac12
\mathbb E\left[
\bigl((1-p_a)^{1-s}+p_a^{1-s}\bigr)
\bigl((1-p_b)^s+p_b^s\bigr)
\right].
\]
Since \(p_a\) and \(p_b\) are i.i.d. copies of a random variable $q$ with distribution $q \sim Q(\rho |Z|)$,
\[
\mathbb E[Z_i^{(ab)}(s)]
=
\frac12
\mathbb E_q\left[(1-q)^{1-s}+q^{1-s}\right]
\mathbb E_q\left[(1-q)^s+q^s\right].
\]
Using the definition,
$
M_\rho(s)
=
\mathbb E_q\left[(1-q)^s+q^s\right],
$
we obtain,
\[
\mathbb E[Z_i^{(ab)}(s)]
=
\frac12 M_\rho(1-s)M_\rho(s)
=
\Psi_\rho(s).
\]
Since the coordinates are i.i.d., averaging over the random geometry and
applying the union bound over the \(K-1\) competitors in~(\ref{eq:orth_pair}) gives,
\begin{equation} \label{eq:fin_pre_orthml}
\Pr (E)
\le
(K-1)\inf_{0\leq s \leq 1}\Psi_\rho(s)^N.
\end{equation}
It remains to identify the optimizer. The function
$
M_\rho(s)
$
is log-convex as a function of $s$; this follows from a standard
log-convexity argument for Chernoff coefficientsusing Hölder's inequality. In particular, fix \(s_0,s_1\in[0,1]\) and
\(\lambda\in[0,1]\), and consider 
$s=(1-\lambda)s_0+\lambda s_1.
$
Then,
\[
M_\rho(s)
=
\mathbb E_q\left[
(1-q)^{s}+q^{s}
\right]
=
\mathbb E_q\left[
\left((1-q)^{s_0}\right)^{1-\lambda}
\left((1-q)^{s_1}\right)^\lambda
+
\left(q^{s_0}\right)^{1-\lambda}
\left(q^{s_1}\right)^\lambda
\right].
\]
By Hölder's inequality, 
\[
M_\rho(s)
\le
\left(
\mathbb E_q\left[(1-q)^{s_0}+q^{s_0}\right]
\right)^{1-\lambda}
\left(
\mathbb E_q\left[(1-q)^{s_1}+q^{s_1}\right]
\right)^\lambda
=
M_\rho(s_0)^{1-\lambda}M_\rho(s_1)^\lambda,
\]
so \(M_\rho (s)\) is log-convex, hence
$
\log \Psi_\rho(s)
$
is also log-convex, and therefore $\Psi_\rho(s)$ is convex in $s$. Since,
$
\Psi_\rho(s)=\Psi_\rho(1-s),
$
it is also symmetric around \(s=1/2\). Therefore, as a convex
and  symmetric function, $\Psi_\rho(s)$  is minimized at its mid-point \(s=1/2\): 
\[
\Psi_\rho(s) \ge \Psi_\rho(1/2) =
\frac12
\left(
\mathbb E\left[\sqrt{1-q}+\sqrt q\right]
\right)^2
=
\frac12
\left(
\mathbb E_Z\left[
\sqrt{1-Q(\rho|Z|)}
+
\sqrt{Q(\rho|Z|)}
\right]
\right)^2,
\]
where $Z\sim\mathcal N(0,1)$. Substituting in the above bound~(\ref{eq:fin_pre_orthml}) and re-arranging gives the corresponding ML exponent and completes the proof. 
\end{proof}

\noindent \emph{Remark.} Theorem~\ref{thm:orth_mle} gives an achievability condition of the same form as  Theorem~\ref{thm:orth_ham}, showing that again $N = O(\log K ) $ hyperplanes are sufficient. Comparing the corresponding multiplicative constants, Jensen's inequality directly shows that $E_{\mathrm{ML}} (\rho) \ge E_{\mathrm{BSC}} (\rho)$. In specific, since the function \(
g(q) = \sqrt{1-q}+\sqrt q
\)
is concave on \([0,1]\), Jense's gives,
\[
\mathbb E_Z\left[
\sqrt{1-Q(\rho|Z|)}
+
\sqrt{Q(\rho|Z|)}
\right]
=
\mathbb E_Z[g(Q(\rho|Z|))]
\le
g( \mathbb E_Z[Q(\rho|Z|)])
=
g(p) = 
\sqrt{1-p}+\sqrt p.
\]
Squaring both sides, multiplying by \(1/2\), and rearranging shows that $E_{\mathrm{ML}} (\rho) \ge E_{\mathrm{BSC}} (\rho)$,
which verifies that the ML Chernoff exponent always gives an improved bound.

\subsection{Order-statistics error formulas} \label{s:ord_stats}

The Chernoff analysis gives clean exponential upper bounds and highlights the
relevant error exponents. These bounds are especially useful for deriving
achievability conditions of the form \(N\gtrsim \log K/E(\rho)\). However, the
union bound over competitors can be conservative, particularly at finite \(K\)
and \(N\). To obtain sharper approximations, we next exploit the order-statistic
structure of the decoding rule: Conditional on the score of the true class, the
scores of the \(K-1\) competing classes are independent, so the probability of
error can be written as the probability that the maximum competitor score
exceeds the true score. This leads to more precise finite-sample formulas,
at the cost of less explicit expressions than the Chernoff exponents.

\subsubsection{Hamming decoding}

The Hamming decoder is,
\[
\widehat a
=
\arg\min_m d_H(B,C_m).
\]

\begin{proposition}
    In the orthogonal hyperplanes setting of Section~\ref{s:orth_descr}, the  probability of error of the Hamming decoder satisfies,
    \begin{equation} \label{eq:bsc_ord}
P_{e}^{\mathrm{Ham}}
=
\sum_{r=0}^N
\binom Nr
p^r(1-p)^{N-r}
\left[
1-
\left(
1-F_N(r)
\right)^{K-1}
\right],
\end{equation}
where $
p
=
\frac1\pi\arctan\frac1\rho
$, and $F_N(r)$ denotes the cumulative distribution function of the symmetric Binomial random variable,
\(
F_N(r)
=
\Pr\left(
\operatorname{Binomial}\left(N,\frac12\right)\le r
\right).
\)
\end{proposition}

\begin{proof}
Let \(a\) denote the true class as before and define,
\[
R_a
:=
d_H(B,C_a).
\]
Conditional on $C_a$ and $B$, the wrong codewords \(C_b\)  are i.i.d. and
uniform on \(\{\pm1\}^N\), so that, 
$
d_H(B,C_b)\sim \operatorname{Binomial}\left(N,1/2\right).
$
For any competitor $b\neq a$, the probability that it does not beat the true class is given by $1 - F_N(R_a)$.
Since the \(K-1\) wrong codewords are conditionally independent, the probability that no wrong class beats or
ties the true class is,
$  \left(1-F_N(R_a)\right)^{K-1}.
$
Averaging over $C_a$ and $B$, or equivalently over \(R_a\), gives
the exact Hamming error probability,
\begin{equation} \label{ord_stat_orth}
\Pr (E)
=
\mathbb E_{R_a}
\left[
1-
\left(
1-F_N(R_a)
\right)^{K-1}
\right].
\end{equation}
It remains to compute the above expectation. Conditional on \(p_{a,1},\dots,p_{a,N}\), the true distance~is,
\[
R_a=\sum_{i=1}^N E_i,
\qquad
E_i\mid p_{a,i}\sim \operatorname{Bernoulli}(p_{a,i})
\]
independently. Therefore its conditional probability-generating function is given by,
\[
\mathbb E[t^{R_a}\mid p_{a,1},\dots,p_{a,N}]
=
\prod_{i=1}^N \left(1-p_{a,i}+p_{a,i}t\right).
\]
Averaging over the i.i.d.  \(p_{a,1},\dots,p_{a,N}\) gives the unconditional probability-generating function, 
\[
\mathbb E[t^{R_a}]
=
\mathbb E\left[
\prod_{i=1}^N \left(1-p_{a,i}+p_{a,i}t\right)
\right]
=
\prod_{i=1}^N
\mathbb E\left[1-p_{a,i}+p_{a,i}t\right].
=
(1-p+pt)^N,
\]
which is the probability-generating function of
\(\operatorname{Binomial}(N,p)\). Hence, unconditionally, we get,
$
R_a\sim \operatorname{Binomial}(N,p),
$
where, as before,
$
p
=
\frac1\pi\arctan\frac1\rho.
$
Substituting appropriately the Binomial distribution in~(\ref{ord_stat_orth}) completes the proof.
\end{proof}

\noindent \textbf{Conclusion.} 
Proposition 1 shows that, perhaps surprisingly, in this setting we are able to obtain an exact, finite-\(N\)
closed-form expression for the Hamming probability of error. This formula is
particularly useful for numerical evaluation, since it can be computed in just
\(O(N)\) operations. Nevertheless, since the finite-sum expression does not directly
expose the scaling of the error probability with \(N\),~\(K\), and \(\rho\),
the Chernoff bound derived above still remains useful: It isolates a simple
exponential rate, gives an interpretable error exponent, and leads directly
to achievability condition of the form \(N\gtrsim \log K/E(\rho)\).

\subsubsection{Maximum likelihood decoding} \label{s:ord_ml}

The log-likelihood score for class \(m\) is,
\[
L_m(B)
=
\sum_{i=1}^N
\left[
\mathbf 1\{B_i=C_{m,i}\}\log(1-p_{m,i})
+
\mathbf 1\{B_i\neq C_{m,i}\}\log p_{m,i}
\right].
\]
The ML decoder chooses,
\[
\widehat a_{\mathrm{ML}}
=
\arg\max_m L_m(B),
\]
and an  ML error occurs if there exists \(b\neq a\) such that
$
L_b(B)\ge L_a(B).
$
Conditional on the true-class score \(L_a\) and on the observed vector
\(B\), the wrong-class scores 
$
L_b(B),
$
are i.i.d. Therefore,
\[
P_e^{\mathrm{ML}}
=
\mathbb E_{L_a,B}
\left[
1-
\left(
1-
\Pr\left(L_b(B)\ge L_a \mid L_a,B\right)
\right)^{K-1}
\right].
\]
For each wrong class \(b\neq a\), the competitor bits \(C_{b,i}\) are fair signs
relative to the fixed observed vector \(B\); and they are independent of
the corresponding probabilities \(p_{b,i}\). Hence, the conditional distribution of the
wrong-class score does not depend on the realized value of \(B\), and,
\[
P_e^{\mathrm{ML}}
=
\mathbb E_{L_a}
\left[
1-
\left(
1-
\Pr\left(L_b\ge L_a\mid L_a\right)
\right)^{K-1}
\right].
\]
Now, it turns out that the distributions of $L_a$ and $L_b$ admit simple  scalar-sum distributions, 
\[
L_a
=
\sum_{i=1}^N S_i, \quad 
L_b
=
\sum_{i=1}^N T_i,
\]
where it remains to define the true-score and wrong score increments $S_i$ and $ T_i$, respectively.
Conditional on the true-class realized probability \(p_{a,i}\), the \(i\)-th true-score
increment is,
\[
S_i
=
\begin{cases}
\log p_{a,i}, & \text{with probability } p_{a,i},\\
\log(1-p_{a,i}), & \text{with probability } 1-p_{a,i},
\end{cases}
\]
where, as before, 
$
p_{a,i}
=
Q(\rho |Z_i|),
$
with $
Z_i\overset{\mathrm{iid}}{\sim}\mathcal N(0,1).
$
So, \(S_1,\dots,S_N\) are i.i.d. scalar increments.
Similarly,
the \(i\)-th wrong-score increment is given by,
\[
T_i
=
\begin{cases}
\log q_i, & \text{with probability } 1/2,\\
\log(1-q_i), & \text{with probability } 1/2,
\end{cases}
\]
where 
$
q_i
=
Q(\rho |Z_i'|),
$
with $
Z_i'\overset{\mathrm{iid}}{\sim}\mathcal N(0,1),
$
independently of the true-class variables. Therefore \(T_1,\dots,T_N\) are i.i.d. and
independent of \(L_a\). So, by defining the wrong-score tail function,
\[
G_N(t)
:=
\Pr\left(
\sum_{i=1}^N T_i \ge t
\right),
\]
the ML probability of error can be written as,
\begin{equation} \label{eq:orth_ml_ord}
P_e^{\mathrm{ML}}
=
\mathbb E_{S_1,\dots,S_N}
\left[
1-
\left(
1-
G_N\left(\sum_{i=1}^N S_i\right)
\right)^{K-1}
\right].
\end{equation}
\textit{Remark.} The above expression can be evaluated by Monte Carlo (MC) over scalar i.i.d. sums, or by
numerical convolution of the one-coordinate score distributions. The details of an
appropriate Monte Carlo sampler are given in Appendix~\ref{app:sampler}; while this computation is
more expensive than the direct Hamming formula, it is still much cheaper than
simulating the full \(K\)-class experiment.
The resulting MC estimate gives sharp finite-$N$ predictions for the probability of error compared, but -- unlike our previous Chernoff analysis -- does not directly expose an exponential scaling law or
lead to a simple closed-form achievability condition.

\subsection{Simulation results}

In this section, we run appropriate simulation experiments to empirically validate the main theoretical findings for the orthogonal hyperplanes setting.
In specific, we inspect the empirical probability of error, together with the corresponding Chernoff upper bounds and order statistics expressions we derived, both for Hamming and ML decoding.
In order to examine which values of the SNR parameter $\rho$ seem practically relevant, in Table~\ref{tab:p-rho} we present a range of values for~$\rho$ together with their corresponding induced crossover probabilities $p =  \frac{1}{\pi}\arctan({1}/{\rho})$, which play a key role in much of the above analysis.

\begin{table}[h!]
    \centering
    \begin{tabular}{ccccccc}
        \toprule
        $\rho$
        & $0.1$ & $0.5$ & $1$ & $2$ & $5$ & $10$ \\
        \midrule
        $p$
        & $0.468$
        & $0.352$
        & $0.250$
        & $0.148$
        & $0.063$
        & $0.032$ \\
        \bottomrule
    \end{tabular}
    \caption{Average crossover probability $p =
        \frac{1}{\pi}\arctan\!\left(\frac{1}{\rho}\right)$  for different values of $\rho$.}
    \label{tab:p-rho}
\end{table}

Here, we focus on the representative cases $\rho=1$ and
$\rho=2$, which correspond to practically relevant intermediate
crossover probabilities. Results for more extreme regimes, where
$p$ is closer either to $1/2$ or to $0$, can be found in Appendix~\ref{app:sims}. These additional plots exhibit the same qualitative behavior
as the cases reported here. The main effect of increasing $\rho$ is to
reduce the crossover probability and shift all error curves towards
smaller values of $N$, so that the same target error probability can be
achieved using fewer hyperplanes.
 In all cases, the results are obtained from $10^4$ repetitions of the experiment, using $K=32$ class labels.

The results presented in Figure~\ref{fig:orth} provide strong empirical validation of the theoretical predictions. First, the derived Chernoff expressions \eqref{eq:orth_ub}-\eqref{eq:fin_pre_orthml} give valid upper bounds on the probability of error for both Hamming and ML decoding, respectively. Although these bounds are somewhat loose at finite $N$, they correctly capture the exponential decay of the error probability and lead to clean achievability conditions of the form,
\[
N\gtrsim \log K/E(\rho),
\]
where $E(\rho)$ is the corresponding error  exponent (Hamming or ML), that depends on the SNR. Second, ML decoding consistently and often substantially outperforms Hamming decoding. This highlights the importance of accounting for the unequal reliability of the individual binary decisions, and shows that treating all decisions equally, as in Hamming decoding or in a majority-voting rule, can be significantly suboptimal in general classification problems.

\begin{figure}[h!]
    \centering
    \includegraphics[width=0.48\linewidth]{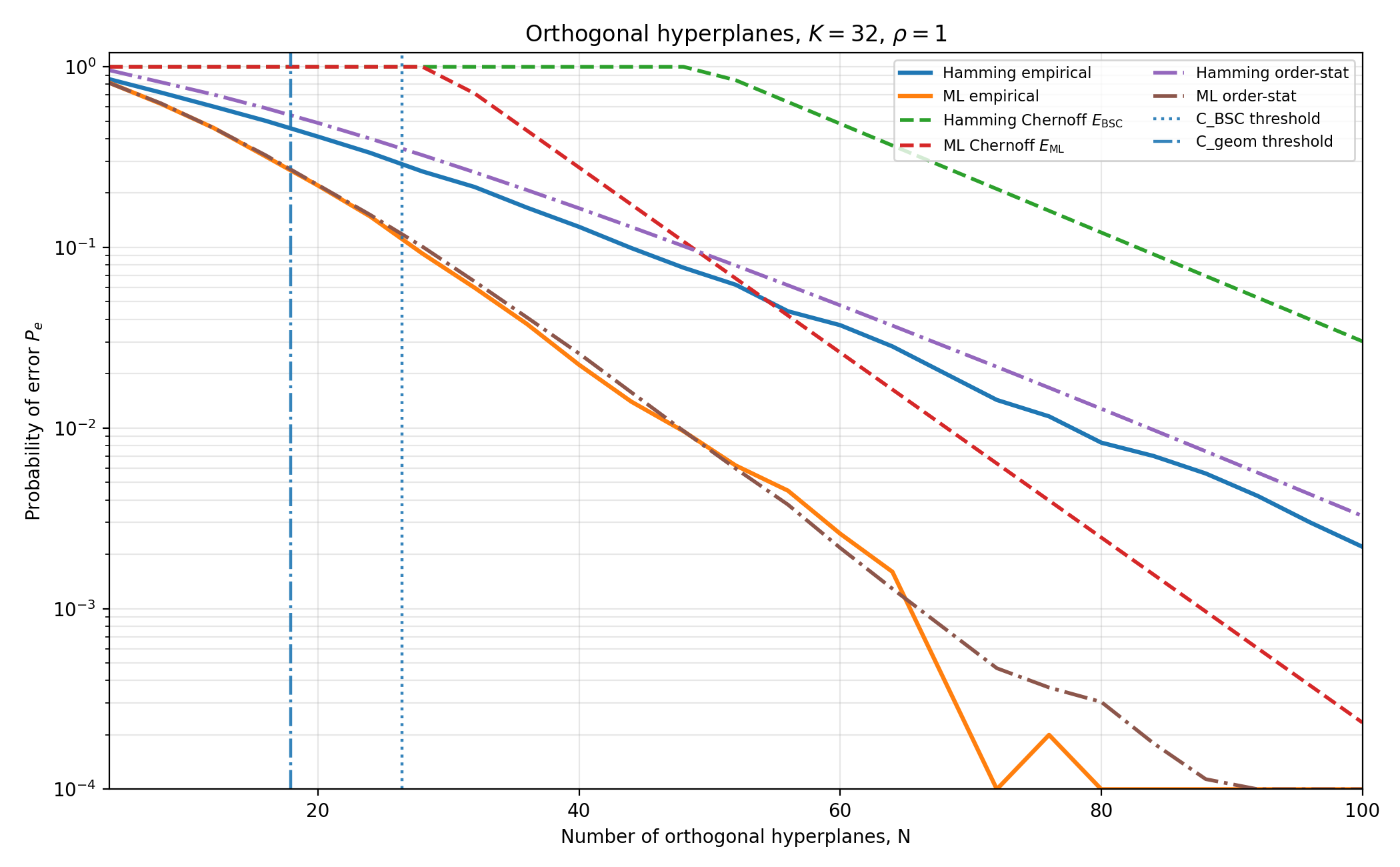}
    ~
        \includegraphics[width=0.48\linewidth]{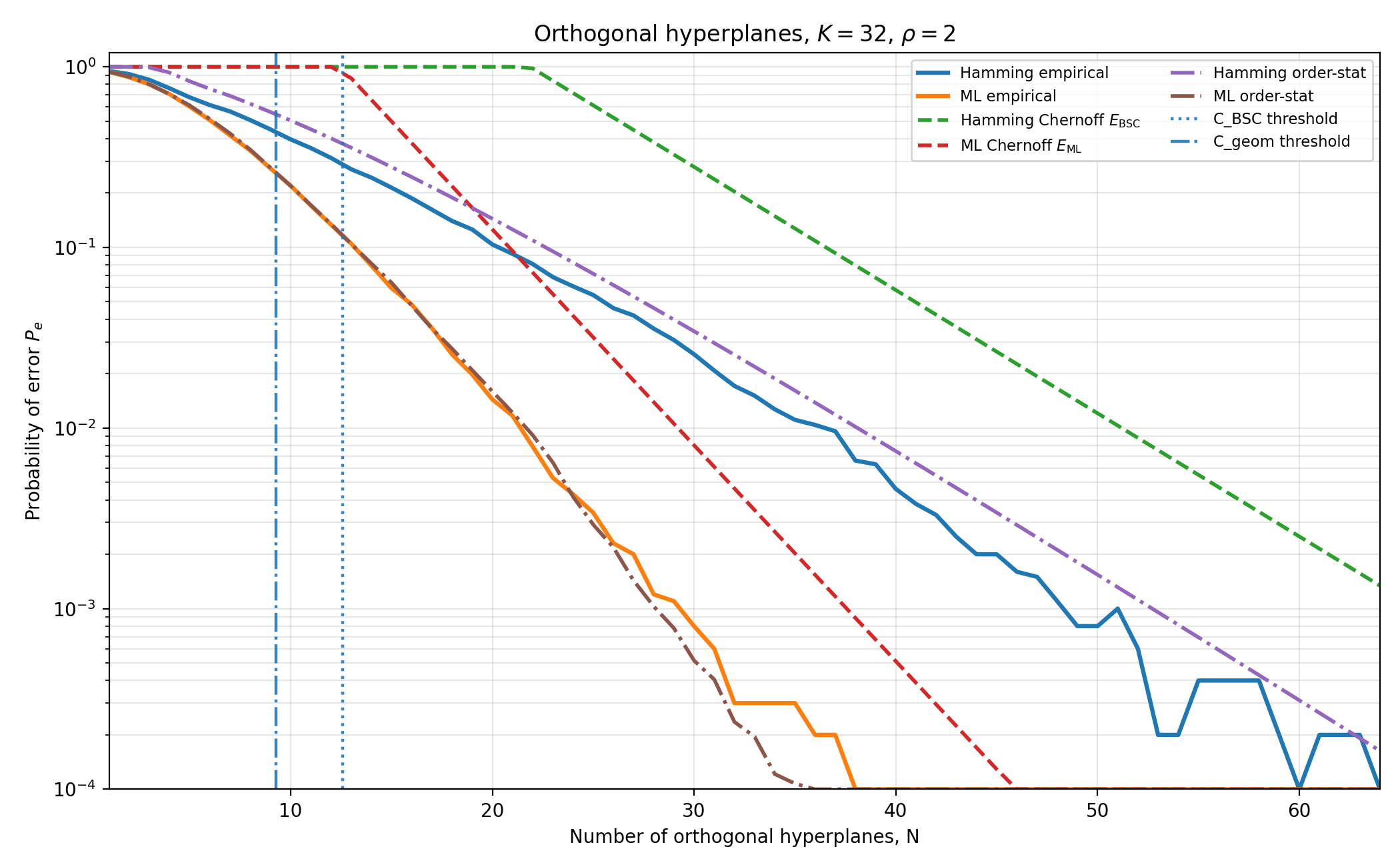}
    \caption{Probability of error as a function of the number of orthogonal hyperplanes. Left: $\rho=1$, Right: $\rho=2$}
    \label{fig:orth}
\end{figure}

Finally, the derived order-statistics expressions closely follow the empirical error curves, providing a more accurate finite-$N$ characterization than the Chernoff bounds and further validating the theory. The small systematic discrepancy observed for the Hamming order-statistic curve is explained by the treatment of ties, since the analytical expression counts every tie with a competing class as an error, whereas the simulations use random tie-breaking. Consequently, the theoretical Hamming order-statistic curve is slightly conservative.

\section{Random hyperplanes with noise} \label{s:random_hyps}

We now return to the original setting of practical interest, in which the
hyperplanes are drawn randomly, and are no longer assumed to be
orthogonal. The analysis of the orthogonal model provides a useful benchmark for
this setting, since in sufficiently high dimensions i.i.d. random
hyperplanes are approximately orthogonal; this high-dimensional regime is studied in detail in Section~\ref{s:high_d_analysis}. Viewed as a communication channel, the
critical difference from the orthogonal setting is that the coordinate bitflips
are no longer~independent.

\subsection{Model description} \label{s:iid_descr}

The centers and test-point are generated exactly as before, with the true class $a$:
\[
x_1,\dots,x_K\overset{\mathrm{iid}}{\sim}\mathcal N(0,\tau^2 I_d),
\qquad
z\sim\mathcal N(0,\sigma^2 I_d),
\qquad
y=x_a+z.
\]
The hyperplanes are now sampled independently as,
\[
u_1,\dots,u_N\overset{\mathrm{iid}}{\sim}\mathrm{Unif}(S^{d-1}).
\]
The binary codeword of \(x_k\), and the observed binary vector are again denoted as,
\[
C_k
=
\left(
\operatorname{sign}(u_1^\top x_k),
\dots,
\operatorname{sign}(u_N^\top x_k)
\right),
\quad
B
=
\left(
\operatorname{sign}(u_1^\top y),
\dots,
\operatorname{sign}(u_N^\top y)
\right).
\]
In this setting, it is convenient to also define the following angles,
\begin{equation} \label{eq:angles}
\alpha := \angle(x_a,y),\qquad
\theta _ b := \angle(x_a,x_b),\qquad
\beta _ b := \angle(x_b,y).
\end{equation}

\subsection{Hamming decoder}

First, we again consider the Hamming decoder,
\[
\widehat a
=
\arg\min_m d_H(B,C_m).
\]
By appropriately applying a Chernoff analysis, we are again able to obtain an upper bound on the probability of error in this setting.

\begin{theorem} \label{thm:iid_ham}
In the random hyperplanes setting of Section~\ref{s:iid_descr}, the probability of error using a Hamming decoder satisfies,
\begin{equation} \label{eq:iid_ham_bound}
\Pr (E)
\le
(K-1)\mathcal B_{d,\rho}(N),
\quad 
\end{equation}
where,
\begin{equation} \label{beta_def}
\mathcal B_{d,\rho}(N)
:=
\mathbb E_{x_a,z,x_b}
\left[
e^{-N E_{\mathrm{iid}}(\theta _b,\alpha,\beta_b)}
\right],
\end{equation}
and,
\begin{equation} \label{eq:eiid}
E_{\mathrm{iid}}(\theta _b,\alpha,\beta_b)
=
\begin{cases}
-\log\left[
1-\dfrac{\theta _ b}{\pi}
+
\dfrac1\pi\sqrt{\theta _ b^2-(\beta _ b-\alpha)^2}
\right],
& \beta _ b\ge\alpha,\\[2ex]
0,
& \beta _b <\alpha,
\end{cases}
\end{equation} 
with the definition of the angles $\alpha, \theta_b , \beta _b$ as in (\ref{eq:angles}).
\end{theorem}

\begin{proof}
Fix a competitor $b \neq a$, and as before define,
\begin{equation}\label{mab_def}
M_i^{(ab)}
:=
\mathbf 1\{B_i\neq C_{b,i}\}
-
\mathbf 1\{B_i\neq C_{a,i}\}.
\end{equation}
Conditional on \(x_a,x_b,z\), the random variables \(M_i^{(ab)}\) are i.i.d. since the hyperplanes \(u_i\) are i.i.d., and their corresponding conditional distribution is given in the following lemma.

\begin{lemma} \label{lem:cond_mab}
Given $x_a,x_b,z$, the conditional law of \(M_i^{(ab)} \in\{-1,0,+1\}\)  is,
\[
a_+ :=\Pr(M_i^{(ab)}=+1\mid x_a,x_b,z)=\frac{\theta _ b+\beta _ b-\alpha}{2\pi},
\]
\[
a_- := \Pr(M_i^{(ab)}=-1\mid x_a,x_b,z)=\frac{\theta _ b-\beta _b +\alpha}{2\pi},
\]
\[
a_0 : = \Pr(M_i^{(ab)}=0\mid x_a,x_b,z)=1-\frac{\theta _b}{\pi},
\]
where the angles $\alpha, \theta_b , \beta _b$ are defined in (\ref{eq:angles}).
\end{lemma}

\begin{proof}
Since \( M_i^{(ab)}\neq0\) exactly when the hyperplane $u_i$ separates \(x_a\) and \(x_b\), by rotational symmetry of the random hyperplane sampling we get,
\begin{equation} \label{pmo}
\Pr(M_i^{(ab)}\neq0 \mid x_a,x_b,z) = a_++a_-=\frac{\theta _ b}{\pi},
\end{equation}
where for simplicity of notation we omit conditioning on $x_a, x_b,z $ for rest of this proof. 
Now, by taking expectations in the definition of $M_i^{(ab)}$ in (\ref{mab_def}), we get, 
\[
\mathbb E[M_i^{(ab)}]=
\Prob(B_i\neq C_{b,i})
-
\Prob(B_i\neq C_{a,i})
=
\Prob(\operatorname{sign}(u_i^\top y)\neq \operatorname{sign}(u_i^\top x_b))
-
\Prob(\operatorname{sign}(u_i^\top y)\neq \operatorname{sign}(u_i^\top x_a)),
\]
and by rotational symmetry this gives,
\begin{equation} \label{pmo2}
 \mathbb E[M_i^{(ab)}] =\frac{\beta _ b}{\pi} - \frac{\alpha}{\pi} = 
 \Pr(M_i^{(ab)}=+1)- 
 \Pr(M_i^{(ab)}=-1)
 = a_+ - a_-   
\end{equation}
Solving the system of equations (\ref{pmo})-(\ref{pmo2}) for $a_+,a_-$  completes the proof of the lemma.
\end{proof}
\noindent Now, for any \(s\ge0\), conditional on $x_a,x_b,z$, a Chernoff bound gives,
\[
\Pr \left(
\sum_{i=1}^N M_i^{(ab)}\le0
\,\middle|\,
x_a,x_b,z
\right)
\le
\mathbb E  \left[
\exp\left(
-s\sum_{i=1} ^ N M_i^{(ab)}
\right)
\middle|\,
x_a,x_b,z
\right ] ,
\]
and since the random variables \(M_i^{(ab)}\) are conditionally independent, from Lemma~\ref{lem:cond_mab}, we get,
\[
\Pr \left(
d_H (B, C_b) \leq d_H (B, C_a)
\,\middle|\,
x_a,x_b,z
\right)
=
\Pr \left(
\sum_{i=1}^N M_i^{(ab)}\le0
\,\middle|\,
x_a,x_b,z
\right)
\leq
\left(
a_+e^{-s}+a_-e^s+a_0
\right)^N
\]
Optimizing over $s$, if \(\beta_b\ge\alpha\), the minimizer is,
\[
s^\star
=
\frac12\log\frac{a_+}{a_-}
=
\frac12
\log
\frac{\theta_b+\beta_b-\alpha}
{\theta_b-\beta_b+\alpha}.
\]
If \(\beta_b\ge\alpha\), we get \(s^\star\ge0\),
and substituting this value gives,
\[
\inf_{s\ge0}
\left(
a_+e^{-s}+a_-e^s+a_0
\right)
=
1-\frac{\theta_b}{\pi}
+
\frac1\pi
\sqrt{\theta_b^2-(\beta_b-\alpha)^2}.
\]
If
\(\beta_b<\alpha\), then $s^* <0$, and the infimum over \(s\ge0\) is attained
at \(s=0\), giving the trivial bound~\(1\).
So, with the definition of $E_{\mathrm{iid}}(\theta _b ,\alpha,\beta_ b )$ from~(\ref{eq:eiid}), we get the pairwise error bound,
\[
\Pr \left(
d_H (B, C_b) \leq d_H (B, C_a)
\,\middle|\,
x_a,x_b,z
\right)
\le
e^{-N E_{\mathrm{iid}}(\theta_b,\alpha,\beta_b)},
\]
and taking a union bound over the competitors $b \neq a$, gives the conditional probability of error~bound,
\begin{equation}
\Pr (E | x_1 ^ K,z)
\le
\sum_{b\neq a}
e^{-N E_{\mathrm{iid}}(\theta_b,\alpha,\beta_b)}.
\end{equation}
Finally, averaging over the random centers $x_1 ^ K$ and the noise $z$, we get the required result by the tower property, since all competitors \(b\neq a\) are identically distributed by symmetry.
\end{proof}

\noindent \textit{Remarks.} Theorem~\ref{thm:iid_ham} gives an upper to the probability of error of Hamming decoding. 
Unlike the orthogonal-hyperplane bound,  $\mathcal B_{d,\rho}(N)$ is not available in closed form, so the result does
not immediately yield an explicit achievability condition of the form \(N\gtrsim \log K/E(\rho)\). Nevertheless, the form of $\mathcal B_{d,\rho}(N)$ allows to evaluate the bound efficiently by direct Monte
Carlo, without requiring to simlulate the \(N\) hyperplanes.
Moreover, by analysing the
asymptotic behavior of \(\mathcal B_{d,\rho}(N)\), we are able to obtain a clean
achievability condition in sufficiently high dimensions.


\subsubsection{Numerical evaluation of the Hamming bound}

First, it is noted that the expectation in~(\ref{beta_def}) directly leads to a simple Monte Carlo estimate for $\mathcal B_{d,\rho}(N)$, by obtaining $T$ Gaussian samples of $x_a,x_b,z$ in $\mathbb R ^d$:
\[
\widehat{\mathcal B}_{d,\rho}(N)
=
\frac1T
\sum_{t=1}^T
e^{-N E_{\mathrm{iid}}(\theta  ^{(t)},\alpha ^{(t)},\beta ^{(t)})},
\]
where $x_a ^ {(t)}, x_b ^ {(t)} \sim \mathcal N (0,  \tau ^ 2 I_d) $, $z ^ {(t)} \sim \mathcal N (0,  \sigma ^ 2 I_d) $, and the corresponding angles $\alpha ^{(t)}, \beta ^{(t)}, \theta ^{(t)}$ are directly obtained from~(\ref{eq:angles}). Importantly, this simple estimate does not require to simulate $N$ hyperplanes. 
Furthermore, it turns out that rotational invariance can be exploited to give an alternative lower-dimensional expression for $\mathcal B_{d,\rho}(N)$, which therefore leads to a low-dimensional Monte Carlo sampler, that avoids $d$-dimensional sampling. This low-dimensional expression is given in the following proposition,  proven in Appendix~\ref{app:proofs}.

\begin{proposition} \label{prop:dim}
For \(d>2\), an alternative expression for $\mathcal B_{d,\rho}(N)$ is,
\begin{equation*}
\mathcal B_{d,\rho}(N)
=
\mathbb E_{R,G,S, V_1,V_2}
\left[
e^{-N E_{\mathrm{iid}}(\theta,\alpha,\beta)}
\right],
\end{equation*}
where the corresponding angles are given by,
\[
\alpha
=
\arccos\left(
\frac{\rho R+G}
{\sqrt{(\rho R+G)^2+S^2}}
\right),
\quad 
\theta=\arccos(V_1),
\quad 
\beta
=
\arccos(V_1\cos\alpha+V_2\sin\alpha),
\]
and the expectation is over,
$
R\sim\chi_d,
G\sim\mathcal N(0,1),
S\sim\chi_{d-1},
$
and $(V_1, V_2)$ with density,
\begin{equation*}\label{comp_pdf}
f_d(v_1,v_2)
=
\frac{d-2}{2\pi}
(1-v_1^2-v_2^2)^{(d-4)/2},
\qquad
v_1^2+v_2^2\le 1.
\end{equation*}
\end{proposition}

\subsubsection{High-dimensional asymptotics: \texorpdfstring{\(d\to\infty\)}{d to infinity}, fixed \texorpdfstring{\(N\)}{N}}

In sufficiently high dimensions, we expect the random hyperplanes to behave as orthogonal hyperplanes due to the concentration of angles. This phenomenon is studied in its full generality in Section~\ref{s:high_d_analysis}. In the following lemma, a much weaker  -- yet useful -- statement is shown, namely, that as $d\to \infty$ we recover the BSC achievability condition from Theorem~\ref{thm:bsc}.

\begin{lemma}
\label{lem:B_high_dim}
For fixed \(N >0\) and \(\rho>0\), as \(d\to\infty\),
\[
\mathcal B_{d,\rho}(N)
\to
e^{-N E_{\mathrm{BSC}}(\rho)},
\]
and therefore a sufficient  condition to ensure $\Pr (E) \leq \delta$ is,
\[
N\ge \frac{\log((K-1)/\delta)}{E_{\mathrm{BSC}}(\rho)}.
\]
\end{lemma}

\begin{proof}
First, consider the  angle
$
\alpha=\angle(x_a,x_a+z),
$
so that,
\[
\cos\alpha
=
\frac{x_a^\top(x_a+z)}
{\|x_a\|\|x_a+z\|},
\]
where
$
x_a\sim \mathcal N(0,\tau^2 I_d),
\ 
z\sim \mathcal N(0,\sigma^2 I_d).
$
By the law of large numbers, we easily get the concentration of Gaussian norms and inner products,
\[
\frac{\|x_a\|^2}{d} \overset{P}{\to} \tau^2,
\qquad
\frac{\|z\|^2}{d} \overset{P}{\to} \sigma^2,
\qquad
\frac{x_a^\top z}{d} \overset{P}{\to}0.
\]
Therefore,
\[
\cos\alpha
\overset{P}{\to}
\frac{\tau}{\sqrt{\tau^2+\sigma^2}}
=
\frac{\rho}{\sqrt{1+\rho^2}}
\]
and since $\arccos$ is continuous,
\[
\alpha
\overset{P}{\to} \alpha_\infty = 
\arccos\left(\frac{\rho}{\sqrt{1+\rho^2}}\right)
=
\arctan\frac1\rho.
\]
Similarly, since \(x_b\) is independent of both \(x_a\) and \(y=x_a+z\), its direction is asymptotically orthogonal to both, and the high-dimensional concentration of angles yields,
\[
\theta_ b=\angle(x_a,x_b) \overset{P}{\to} \frac{\pi}{2},
\qquad
\beta _ b=\angle(x_b,y) \overset{P}{\to} \frac{\pi}{2}.
\]
Since \(\rho>0\), we have
$
\alpha _ \infty = \arctan 1 / \rho < \pi/2,
$
and hence $\beta _b > \alpha$.
Therefore,
\[
e^{-N E_{\mathrm{iid}}(\theta_b,\alpha,\beta_b)}
\overset{P}{\longrightarrow}
e^{-N E_{\mathrm{iid}}
\left(
\frac{\pi}{2},
\alpha_\infty,
\frac{\pi}{2}
\right)},
\]
and since $e^{-N E_{\mathrm{iid}}(\theta_b,\alpha,\beta_b)}$ is uniformly bounded, we also get convergence in expectation, 
\[
\mathcal B_{d,\rho}(N) =  \mathbb  E _{x_a,x_b,z} [e^{-N E_{\mathrm{iid}}(\theta_b,\alpha,\beta_b)} ] 
\to
e^{-N E_{\mathrm{iid}}
\left(
\frac{\pi}{2},
\alpha_\infty,
\frac{\pi}{2}
\right)}.
\]
Substituting in~(\ref{eq:eiid}), and noting that 
$
p=\frac{\alpha_\infty}{\pi}
=
\frac1\pi\arctan \frac1\rho $, completes the proof, since it finally yields, $E_{\mathrm{iid}}
\left(
\frac{\pi}{2},
\alpha_\infty,
\frac{\pi}{2}
\right) = E_{\mathrm{BSC}}(\rho)$.
\end{proof}

\subsubsection{Large-\texorpdfstring{\(N\)}{N} asymptotics: \texorpdfstring{\(N\to\infty\)}{N to infinity}, fixed \texorpdfstring{\(d\)}{d}}

For fixed \(d\), as \(N\to\infty\), Hamming distances converge to angular distances. In particular,
\[
\frac1N d_H(B,C_m) = \frac1 N \sum _{i=1} ^N \mathbf 1\{B_i\neq C_{m,i}\}
\overset{P}{\to}
\frac{\angle(y,x_m)}{\pi},
\]
The convergence follows from the law of large numbers, since conditional on $x_m,y$ the indicator random variables are i.i.d. Bernoulli with $p = {\angle(y,x_m)}/{\pi}$.
Therefore, for any fixed realization of $x_1 ^K, z$, the Hamming decoder converges to angular nearest-neighbor decoding:
\[
\widehat a_{\mathrm{Ham}}
\to
\arg\min_m \angle(y,x_m).
\]
Therefore, for any fixed $d$, the limiting probability of error is given by,
\[
P_e ^ {\infty} = 
\Pr\left(
\min _ {b\neq a} \beta_b < \alpha
\right).
\]
In the following proposition, we derive an exact and numerically tractable
expression for the limiting probability of error, which can be evaluated
efficiently by low-dimensional Monte Carlo. Its proof follows an order statistics argument, and is given in Appendix~\ref{app:proofs}.

\begin{proposition}
\label{prop:angular_floor}
In the random hyperplanes setting of Section~\ref{s:iid_descr}, for the probability of error of Hamming decoding, as
\(N\to\infty\) we get,
\begin{equation} \label{eq:lim_error}
P_e^{\mathrm{Ham}}
\to
P_e ^ {\infty} = 
\Pr\left(
\min_{b\neq a}\beta_b\le\alpha
\right)
=
\mathbb E_{\alpha}
\left[
1-
\left(
I_{\frac{1+\cos\alpha}{2}}
\left(
\frac{d-1}{2},
\frac{d-1}{2}
\right)
\right)^{K-1}
\right],
\end{equation}
where \(I_x(a,b)\) denotes the regularized incomplete beta function, and,
\[
\cos\alpha
=
\frac{\rho R+G}
{\sqrt{(\rho R+G)^2+S^2}},
\]
with
$
R\sim\chi_d,\
G\sim\mathcal N(0,1),\
S\sim\chi_{d-1},
$
mutually independent.
\end{proposition}

\noindent \textbf {Remark 1.}
For every finite \(d\ge2\) and finite \(\rho<\infty\),
the limiting probability of error is strictly positive.
Moreover,  since $I_x(a,b)$  is strictly increasing in~x, the limiting probability of error is strictly decreasing with~\(\rho\). In particular, as \(\rho\to\infty\),
\(
P_e ^ {\infty}\to0,
\)
and as \(\rho\to0\), $
P_e ^ {\infty}\to 1-1/K.
$

\medskip

\noindent \textbf {Remark 2.}
For any \(\rho>0\),  $P_e ^ {\infty} \to 0$ as $d \to \infty$, showing that in high dimensions the limiting probability of error vanishes. In fact, a more refined asymptotic analysis of the above expression (e.g., via Stirling's approximation) suggests that the decay is exponential in $d$.



\subsection{Maximum likelihood decoding} \label{s:iid_ml}

As before, Hamming decoding is not optimal in this setting, so we can again consider ML decoding to improve performance. Let, 
\[
U=
\begin{bmatrix}
u_1^\top\\
\vdots\\
u_N^\top
\end{bmatrix}
\in\mathbb R^{N\times d},
\qquad
B=\operatorname{sign}(Uy)\in\{\pm1\}^N.
\]
Under hypothesis \(m\),
\[
y=x_m+z,
\qquad
z\sim\mathcal N(0,\sigma^2 I_d).
\]
so that,
$
Uy
=
Ux_m+Uz,
$ and finally,
\[
Uy
\sim
\mathcal N(Ux_m,\sigma^2UU^\top).
\]
Therefore, the likelihood of a candidate class \(m\) is,
\[
P_m(B)
=
\Pr\left(
\operatorname{sign}(G_m)=B
\right),
\qquad
G_m\sim\mathcal N(Ux_m,\sigma^2UU^\top).
\]
In general,
$
UU^\top\neq I_N,
$
so the likelihood does not factor over coordinates as in the orthogonal hyperplanes setting.
So computing the likelihood
\(P_m(B)\) corresponds to computing an \(N\)-dimensional correlated
Gaussian orthant probability. 

The numerical evaluation of Gaussian orthant probabilities is a well-studied problem in the statistics literature, which becomes excessively hard as the dimension grows; see, e.g.,~\cite{ridgway2016computation}. Numerous approaches have been proposed, including numerical integration, importance-sampling, and sequential Monte Carlo methods~\cite{ridgway2016computation,genz1992numerical,hajivassiliou1996simulation,geweke1991efficient}. Although these methods can be effective in moderate dimensions (typically up to few tens, depending on the specific problem), in high dimensions the problem remains challenging. In our setting,  ML decoding becomes already impractical for moderate values of $N$ around $N\ge 50$; see also the simulations section below.

For this reason, we introduce the following approximate ML (`pseudo-ML') decoder, by assuming that bitflips are independent across coordinates. In the next section, we show that in sufficiently high dimensions, this decoder effectively achieves the same performance as the true ML decoder, while being computationally efficient. This observation is empirically validated in the simulations section, where we illustrate the practical utility of the pseudo-ML decoder.

\subsubsection{Approximate pseudo-ML decoder}

The pseudo-ML decoder assumes that bitflips are independent, as in the orthogonal hyperplanes setting. As before, for candidate \(m\), we define the realized crossover probabilities,
\[
p_{m,i}
=
Q\left(
\frac{|u_i^\top x_m|}{\sigma}
\right),
\]
for $i=1,...,N$. The pseudo-likelihood score of candidate $m$ is defined exactly as in the orthogonal hyperplanes setting, by assuming that bitflips are independent across coordinates,
\[
\widetilde P_m(B)
=
\prod_{i=1}^N
\left[
(1-p_{m,i})\mathbf 1\{B_i=C_{m,i}\}
+
p_{m,i}\mathbf 1\{B_i\neq C_{m,i}\}
\right],
\]
and the approximate pseudo-ML decoder is given by,
\[
\widehat m_{\mathrm{pML}}
=
\arg\max_m \widetilde P_m(B).
\]
This decoder is exactly ML in the orthogonal hyperplanes setting.
For random hyperplanes, it is a valid reliability-aware mismatched decoder, but not the optimal one, since it
uses the realized probabilities \(p_{m,i}\) but ignores cross-bit correlations. In high dimensions, where random hyperplanes are approximately orthogonal, we expect this decoder to actually perform similarly with the true ML decoder; this is studied in detail in the next section.

It is noted that Chernoff analysis allows to obtain an upper bound to the probability of error of the pseudo-ML decoder; see Appendix~\ref{app:proofs} for more details. However, the resulting bound does not directly give a clean achievability condition for the required number of hyperplanes, and does not simplify enough to allow for an efficient  Monte Carlo estimate.

\vspace*{-0.15 cm}

\subsection{Theoretical justification of the high-dimensional orthogonal limit } \label{s:high_d_analysis}

\vspace*{-0.05 cm}

As described before, in sufficiently high dimensions we expect random hyperplanes to behave similarly with orthogonal hyperplanes, which provides the main motivation for the careful analysis of the orthogonal-hyperplanes setting in previous sections. We now provide strong theoretical justifications of this high-dimensional orthogonal limit. 

In terms of quantifying the modes of this convergence, we consider the following relevant regimes: a) Covariance-level convergence of the projected noise vector, and b) Full block-likelihood convergence of the observed binary vector. In general, the first condition provides a practical rule of thumb for when to expect random hyperplanes to behave as approximately orthogonal, while the second gives a strong sufficient condition ensuring that the complete observation model has converged in a strict sense. These theoretical findings are also empirically validated in the following experimental section.

\vspace*{-0.2 cm}

\subsubsection{Covariance-level convergence}

With notation as before, the exact sign likelihood under class \(m\) is,
\[
P_m^{}(B)
=
\Pr\left\{
\operatorname{sign}(G_m)=B
\right\},
\qquad
G_m\sim
\mathcal N(Ux_m,\sigma^2UU^\top).
\]
The first sense in which random hyperplanes approximate the orthogonal model is
through the projected-noise covariance, i.e., when,
\[
UU^\top\approx I_N.
\]
Since likelihood-based
decoding depends on the full \(N\)-dimensional Gaussian covariance, this matrix approximation  is naturally measured in operator norm. In the following, \(\|\cdot\|_{\mathrm{op}}\) denotes the spectral norm, or equivalently, the largest singular value of the matrix.

\begin{proposition}
\label{prop:covariance_limit}
Let
$
u_1,\dots,u_N\overset{\mathrm{iid}}{\sim}\operatorname{Unif}(S^{d-1}),
\quad
U=
\begin{bmatrix}
u_1^\top\\
\vdots\\
u_N^\top
\end{bmatrix}
\in\mathbb R^{N\times d}.
$
If \(N/d\to0\), then,
\[
\left\|UU^\top-I_N\right\|_{\mathrm{op}}
=
O_p \left(
\sqrt{\frac Nd}
\right)
\overset{P}{\longrightarrow}0.
\]
\end{proposition}

\begin{proof}
The proof follows directly from standard random matrix theory. For example, by defining 
$A:=\sqrt d\,U^\top\in\mathbb R^{d\times N},
$
its columns are independent, centered, isotropic
sub-Gaussian random vectors in \(\mathbb R^d\), with norms
$
\|A_i\|_2=\sqrt d
$, almost surely. Also,
under \(N/d\to0\), we have \(N\le d\) eventually, and
standard Gram-matrix concentration for such matrices gives, 
\[
\left\|
\frac1dA^\top A-I_N
\right\|_{\mathrm{op}}
=
O_p \left(
\sqrt{\frac Nd}
\right),
\]
e.g., by Theorem~5.58 and Remark~5.59 of~\cite{vershynin2012introduction}.
The required result  follows directly.
\end{proof}

\noindent \textit{Remark.} Thus,
$
 d \gg N
$
is the natural covariance-level regime in which random hyperplanes behave
approximately as orthogonal; this is also validated in the simulations section.
This covariance convergence is, however, only a first-order notion of
approximation. The stronger notion of  convergence of the full block-likelihood of the $N$-bit vector is studied next.

\subsubsection{Full block-likelihood convergence}

A natural strong notion of convergence to the orthogonal limit is convergence of the full block-likelihood in total variation. This is a natural notion of convergences since it ensures that every decoder has asymptotically the same probability of error probability under the two observation models (i.e., under orthogonal or random hyperplanes). In specific, for two
probability distributions \(P\) and \( \widetilde P\) on \(\{\pm1\}^N\), and any decoder rule \(\widehat m:\{\pm1\}^N\to\{1,\dots,K\}\),
\[
\left|
\Pr_P(\widehat m(B)\neq m)
-
\Pr_{\widetilde P}(\widehat m(B)\neq m)
\right|
\le
d_{\mathrm{TV}}(P,\widetilde P).
\]
In our setting, for any observed binary vector $B \in \{\pm1\}^N$, define the likelihood of class $m$ as,
\begin{equation} \label{eq:true_lik}
P_m(B)
=
\Pr\left\{
\operatorname{sign}(G_m)=B
\right\},
\qquad
G_m\sim
\mathcal N(Ux_m,\sigma^2UU^\top),
\end{equation}
and the corresponding pseudo-likelihood as,
\begin{equation} \label{eq:pseudo_lik}
\widetilde P_m (B)
=
\Pr\left\{
\operatorname{sign}( \widetilde G_m)=B
\right\},
\qquad
\widetilde G_m\sim
\mathcal N(Ux_m,\sigma^2 I _N).
\end{equation}

\begin{proposition}
\label{prop:block_likelihood_limit}
Let $P_m$ and $\widetilde P_m$ denote the likelihood and pseudo-likelihood of class $m$, as defined in (\ref{eq:true_lik})-(\ref{eq:pseudo_lik}), respectively.
If $
{N^2}/{d}\to0,
$
then for every class \(m\),
\[
d_{\mathrm{TV}}(P_m,\widetilde P_m)
=
O_p\left(\frac{N}{\sqrt d}\right)
\overset{P}{\longrightarrow}0.
\]
\end{proposition}

\begin{proof}
By Pinsker's inequality, it is sufficient to consider the relative entropy between $P_m$ and~$ \widetilde P_m$.
Set
$
\Sigma:=UU^\top, \
\mu_m:=Ux_m.
$
By the data-processing inequality for relative entropy, applied to the
coordinatewise sign map,
\[
D_{\mathrm{KL}}(P_m\|\widetilde P_m)
\le
D_{\mathrm{KL}}\!\left(
\mathcal N(\mu_m,\sigma^2\Sigma)
\middle\|
\mathcal N(\mu_m,\sigma^2I_N)
\right).
\]
Since the two Gaussian distributions have the same mean,
\[
D_{\mathrm{KL}}\!\left(
\mathcal N(\mu_m,\sigma^2\Sigma)
\middle\|
\mathcal N(\mu_m,\sigma^2I_N)
\right)
=
\frac12
\left[
\operatorname{tr}(\Sigma)-N-\log\det\Sigma
\right] =  - \frac12 \log\det\Sigma,
\]
since every diagonal entry of \(\Sigma\) equals one, and therefore
\(\operatorname{tr}(\Sigma)=N\). 

Let $\Delta : = \Sigma - I_N$.
By Proposition~\ref{prop:covariance_limit}, in this regime we have
$
\|\Delta\|_{\mathrm{op}}
\overset{P}{\longrightarrow}0,
$
so we can apply a relevant Taylor expansion to the log-determinant term,
\[
\log\det \Sigma =
\log\det(I_N+\Delta)
=
\operatorname{tr}(\Delta)
-\frac12\operatorname{tr}(\Delta^2)
+
O\left(
\|\Delta\|_{\mathrm{op}}\|\Delta\|_{\mathrm F}^2
\right),
\]
which follows by applying the Taylor expansion of $\log (1+x)$ 
to the eigenvalues of \(\Delta\).
Since
$
\operatorname{tr}(\Delta)
=
\operatorname{tr}(\Sigma)-N
=
0
$,
and \(\Delta\) is symmetric, we obtain,
\[
D_{\mathrm{KL}}
\left(
\mathcal N(\mu_m,\sigma^2\Sigma)
\,\middle\|\,
\mathcal N(\mu_m,\sigma^2I_N)
\right)
=
\frac14\|\Delta\|_{\mathrm F}^2
+
O\left(
\|\Delta\|_{\mathrm{op}}\|\Delta\|_{\mathrm F}^2
\right).
\]
Moreover, by Proposition~\ref{prop:covariance_limit}, we can bound the Frobenius norm as,
\[
\|\Delta\|_{\mathrm F}^2
\le
N\|\Delta\|_{\mathrm{op}}^2
=
O_p\left(\frac{N^2}{d}\right),
\]
and since \(\|\Delta\|_{\mathrm{op}}=o_p(1)\), the remainder is of smaller
order, so,
\[
D_{\mathrm{KL}}(P_m\|\widetilde P_m)
\le
D_{\mathrm{KL}}
\left(
\mathcal N(\mu_m,\sigma^2\Sigma)
\,\middle\|\,
\mathcal N(\mu_m,\sigma^2I_N)
\right)
=
O_p\left(\frac{N^2}{d}\right).
\]
By Pinsker's inequality, we finally get,
\[
d_{\mathrm{TV}}(P_m,\widetilde P_m)
\le
\sqrt{\frac12D_{\mathrm{KL}}(P_m\|\widetilde P_m)}
=
O_p\left(\frac{N}{\sqrt d}\right).
\]
\end{proof}

\begin{corollary}
\label{cor:pml_optimality}
Let \(P_e^{\mathrm{ML}}\) and \(P_e^{\mathrm{pML}}\) denote the error
probabilities of the true-ML and pseudo-ML decoders, respectively. If $N^2/d \to 0$, then,
\[
0
\le
P_e^{\mathrm{pML}}-P_e^{\mathrm{ML}}
=
O_p\left(\frac{N}{\sqrt d}\right)
\overset{P}{\longrightarrow}0.
\]
\end{corollary}

\begin{proof}
The result follows immediately from the optimality of the true-ML decoder under
\(P_m\), the optimality of pseudo-ML decoder under
\(\widetilde P_m\), and the fact that the error probability of any fixed
decoder changes by at most the total-variation distance between the two
observation laws.  
\end{proof}

\paragraph{Conclusion.}
Hence, $ d \gg N ^2$ provides a strong sufficient condition, ensuring convergence, in total variation, of the  full $N$-bit likelihood to its orthogonal approximation. In this high-dimensional regime, the proposed pseudo-ML decoder, that does not take into account correlations among the coordinate bitflips, is proven to be essentially optimal, while remaining computationally efficient. 
More generally, the detailed orthogonal-hyperplane analysis conducted in the previous section becomes asymptotically
relevant in this high-dimensional regime, including the corresponding
error exponents,  achievability conditions and so on. 

In the following section, we conduct simulation experiments to validate the above theoretical findings. In practice, it turns out that the condition $ d \gg N^2 $ appears to be  conservative, in the sense that pseudo-ML decoding can perform nearly as well as true-ML at much
lower dimensions. 
In fact, as a general rule of thumb, it appears that having a number of dimensions comparable to the number of hyperplanes seems to be enough for the two decoders to achieve similar error probabilities.
These empirical findings are illustrated in detail in the next section.

\subsection{Simulation results}

\vspace*{-0.1  cm}

In this section, we conduct simulation experiments to validate the preceding theoretical results, investigate the empirical performance of Hamming and pseudo-ML decoding, and obtain the practical high-dimensional regime in which pseudo-ML essentially performs as true-ML. All Hamming and pseudo-ML results are obtained from $10^4$ repetitions of the experiment. For true-ML, in order to compute the high-dimensional Gaussian orthant probabilities we use the Geweke-Hajivassiliou-Keane (GHK) algorithm~\cite{genz1992numerical,geweke1991efficient,hajivassiliou1996simulation}; 
see also our relevant discussion in Section~\ref{s:iid_ml}. Since this computation quickly becomes excessively expensive, we evaluate the performance of true-ML up to
\(N\approx30\)--\(40\), using \(10^3\) repetitions of each experiment.

In Figure~\ref{fig:iid}, the experimental results are presented for  a setting with \(K=32\) classes in three
representative dimensional regimes: \(d=5\), \(d=20\), and \(d=100\),
corresponding respectively to low, intermediate, and high dimensions.
Overall, the results provide empirical validation of the theoretical
predictions of the previous sections,
and strongly support the use of the computationally efficient pseudo-ML decoder.

\begin{figure}[h!]
\vspace*{-0.2  cm}
    \centering
    \includegraphics[width=0.48\linewidth]{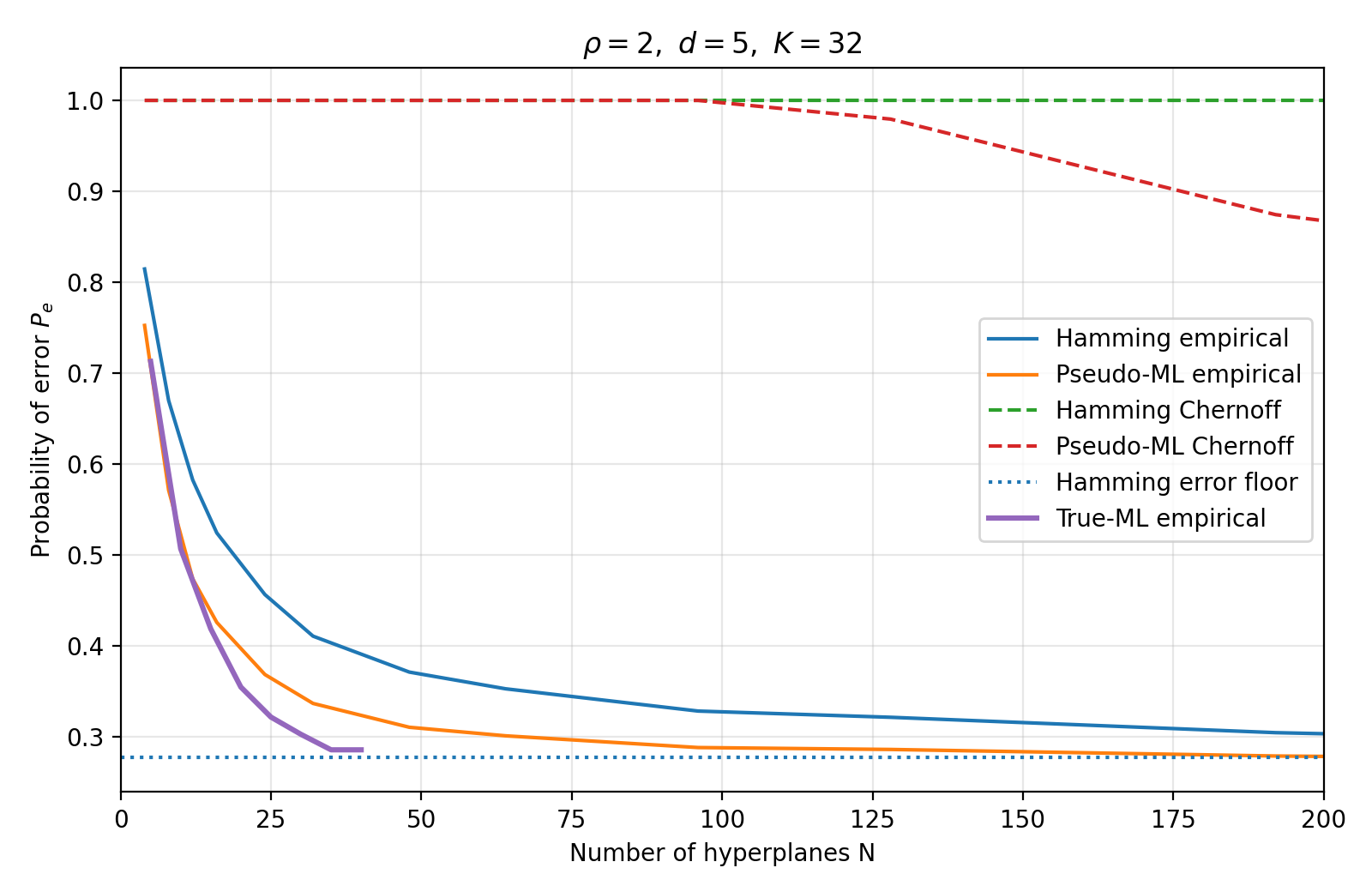}
    ~
        \includegraphics[width=0.48\linewidth]{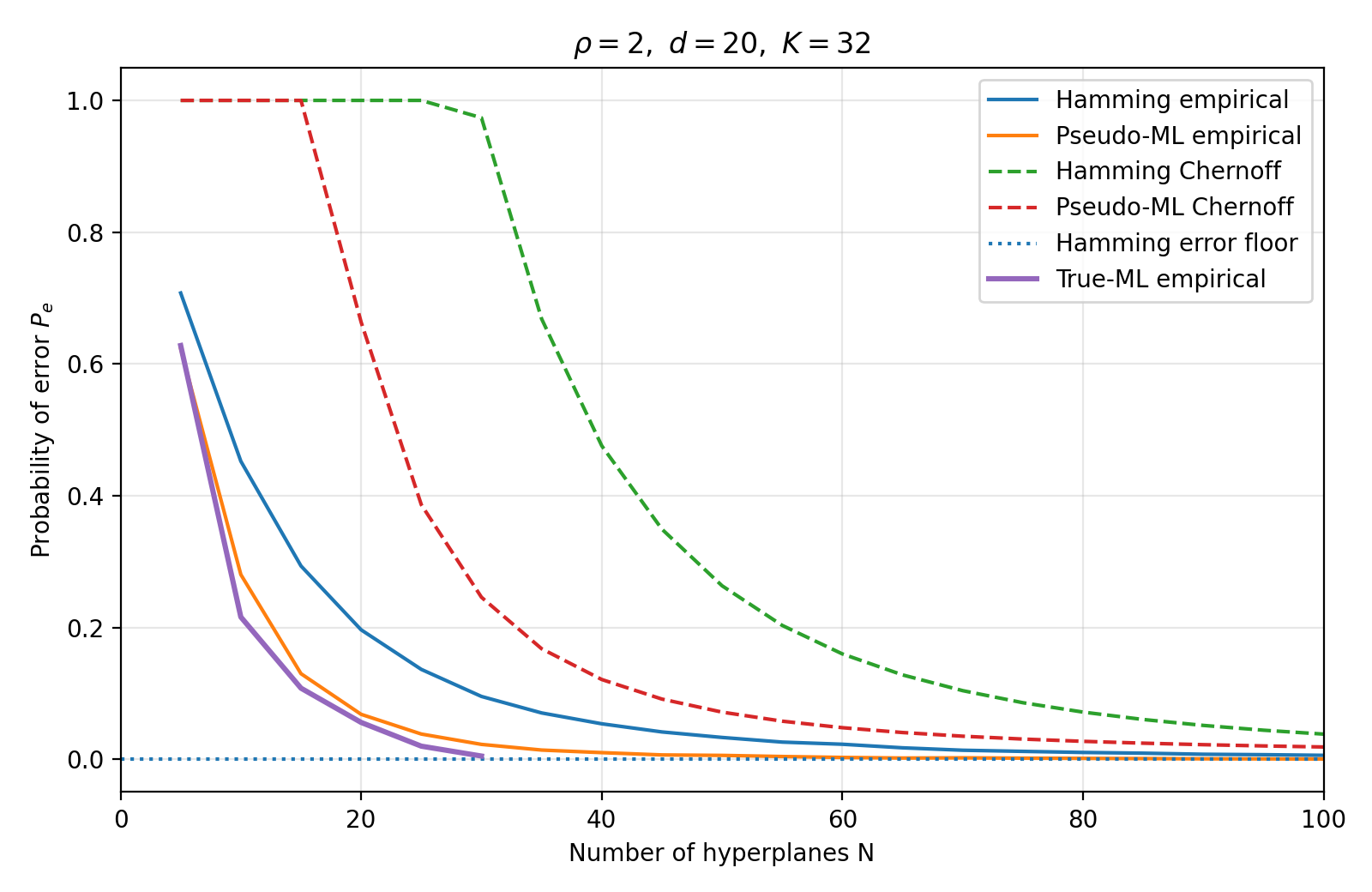}
    ~
        \includegraphics[width=0.48\linewidth]{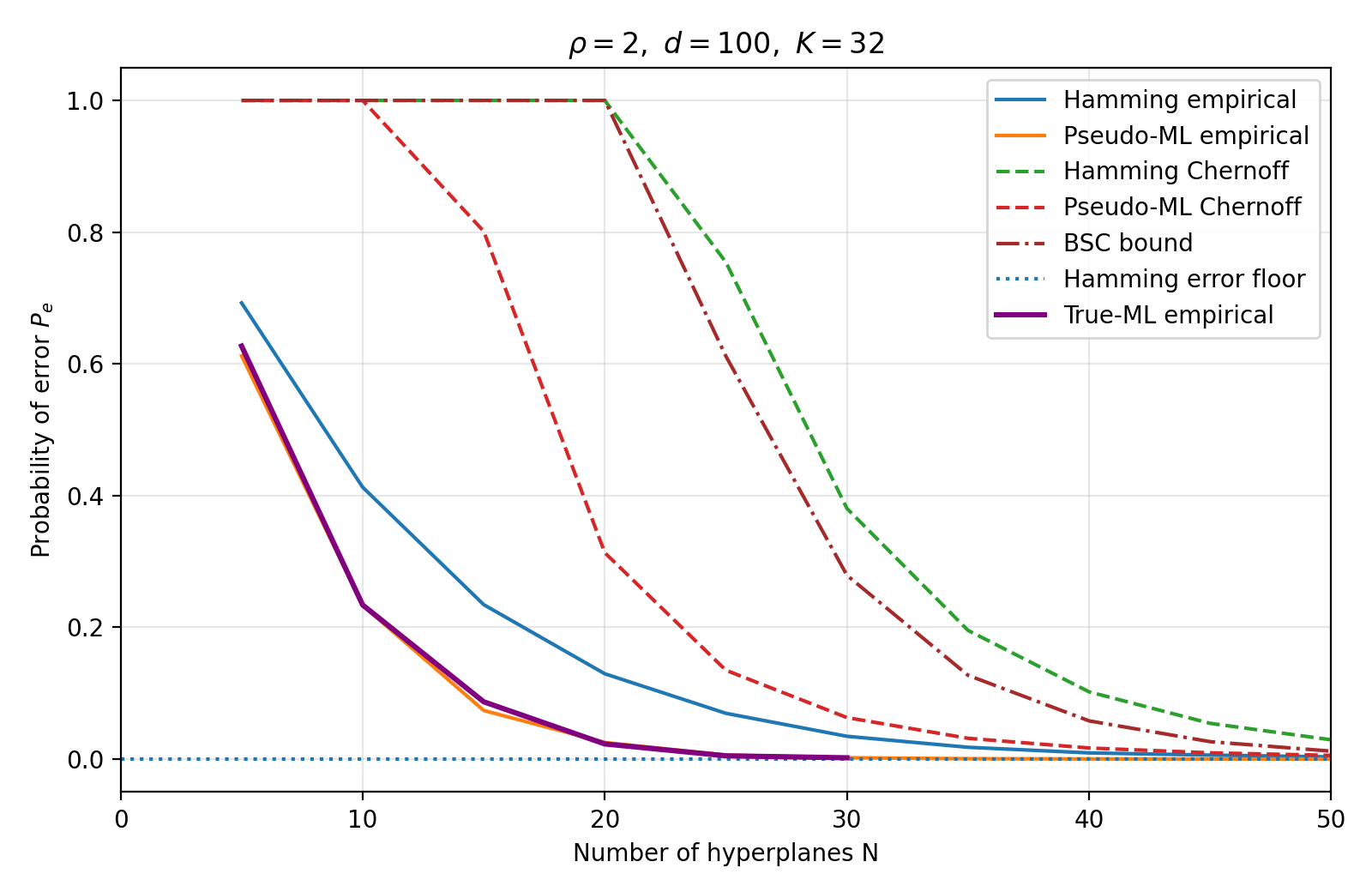}
        \vspace*{-0.3  cm}
    \caption{Probability of error as a function of the number of hyperplanes. Top Left: $d=5$,  Top Right: $d=20$, Bottom: $d=100$.}
    \label{fig:iid}
\end{figure}

Figure~\ref{fig:iid} clearly illustrates that pseudo-ML can in practice perform nearly  as well as true-ML even in moderately low dimensions, despite being computationally much more efficient. 
In particular, the results  suggest that
pseudo-ML
closely tracks the performance of true-ML even when the ambient dimension is
only of the same order as the number of hyperplanes. Specifically, only in the very low-dimensional example  $d=5$ there seems to be an apparent gap between their corresponding error curves, and this happens only after  $N \ge15$.  In moderate and high dimensions where  $d=20$ and $d=100$, pseudo-ML essentially achieves identical performance with true-ML all the way up to $N=30$, where both decoders seem to converge to their limiting probability of error. 
Thus, the theoretical
condition \(d\gg N^2\), which guarantees convergence of the full-block
likelihood in total variation, appears to be conservative in practice, and pseudo-ML decoding can perform essentially as true-ML at much lower dimensions.

As predicted from Proposition~\ref{prop:angular_floor}, as $N$ grows large, the probability of error of Hamming decoding converges to its derived limit in~(\ref{eq:lim_error}), which is also plotted in Figure~\ref{fig:iid}. In addition, it is observed that this limiting error floor decreases rapidly with
the ambient dimension, with the simulations exhibiting the exponential decay
predicted by the theory. 
The high-dimensional BSC limit of Hamming decoding is also empirically validated: For $d=100$, we plot the BSC error bound from~(\ref{eq:orth_ub}), which closely matches the corresponding Chernoff bound~(\ref{eq:iid_ham_bound}), therefore verifying the results of Lemma~\ref{lem:B_high_dim}. 
The fact that across all tested regimes, Hamming decoding
performs substantially worse than pseudo-ML, illustrates the
importance of exploiting the heterogeneous coordinate reliabilities, and shows that treating all binary decisions with equal weight can be significantly suboptimal in general classification problems.

In all the above, we have again focused on the representative case $\rho = 2$, corresponding to a practically relevant moderate-SNR regime. Additional experiments for smaller and larger SNR values can be found in Appendix~\ref{app:sims},
with the corresponding plots essentially displaying the same qualitative behavior. As expected, the main effect of increasing \(\rho\) is that it shifts the error curves towards
smaller values of $N$ and it lowers the corresponding limiting error floor (for a certain value of $d$);
while decreasing \(\rho\) has exactly the opposite effect.

\section{Concluding remarks}

We studied a distributed formulation of multiclass classification, motivated by
settings in which a population of agents, each implementing a lightweight
classifier trained on limited local data, cooperates to produce an
accurate multiclass decision.
To this end, we proposed a random hyperplanes classification scheme and derived explicit bounds for its performance in a variety of settings.
In a noiseless Gaussian setting, we showed that, in sufficiently high dimensions, 
$N = 2\log _2 K$ random hyperplanes are required to separate the $K$ class centers, which is within a factor of two of the information-theoretic minimum. Also, comparing with a natural ECOC alternative, we concluded that the dimensionality requirements of the proposed scheme are much lower, requiring a dimension that scales only logarithmically with $K$.

In a more realistic setting where observations are corrupted by Gaussian noise, we again derived explicit performance bounds for the random hyperplanes classifier. It was shown that in sufficiently high dimensions, $N = O(\log K )$ hyperplanes are again sufficient, with a multiplicative constant that depends on the signal to noise ratio and the decoding mechanism. In specific, it was shown that Hamming decoding, which assigns equal weight to all binary decisions, can be significantly suboptimal compared to reliability-aware decoders. In this view, we proposed a computationally efficient approximate maximum likelihood algorithm, that is both theoretically justified and was found to perform nearly optimally in practice.

Closing, we note a few possible directions for future work. First, it is natural to study an even more realistic
version of the problem, where the class centers are not directly observed and need to be estimated from the data. Also, one can
move beyond hyperplanes and consider more flexible base classifiers at the
local agents. Finally, relaxing the Gaussian
assumptions would help identify which aspects of the theory persist for more
general data distributions.
These extensions would greatly broaden the
scope and practical applicability of the proposed~framework.

\newpage

\bibliographystyle{plain}
\bibliography{refs2}

\newpage

\appendix

\begin{center}
{\huge
{\bf Appendix}\\
}
\end{center}

\section{Theoretical results} \label{app:proofs}

\subsection{Proof of Proposition \ref{prop:dim}}

Because the Gaussian model is rotationally invariant, we may rotate coordinates so that $e_1$ is aligned with the true center $x_a$, i.e.,
$
e_1 = {x_a}/{\|x_a\|}.
$
In these coordinates we may write,
\[
x_a=\tau R e_1,
\]
with $R\sim \chi_d$
since \(x_a\sim \mathcal N(0,\tau^2 I_d)\), where $\chi _ d$ denotes the chi distribution with $d$ degrees of freedom.
 Now we can decompose the noise \(z\) into its component parallel to \(x_a\) and its orthogonal component:
\[
z
=
\sigma G e_1+\sigma g_\perp,
\]
where
$
G\sim\mathcal N(0,1),
$ and $
g_\perp\sim\mathcal N(0,I_{d-1})
$
in the subspace orthogonal to \(e_1\). If $
S:=\|g_\perp\|,
$
then,
$
S\sim \chi_{d-1},
$
and $R,G,S$ are
 independent.
With \(\rho=\tau/\sigma\), we have,
\[
y=x_a+z
=
\tau R e_1+\sigma G e_1+\sigma g_\perp
=
\sigma\left((\rho R+G)e_1+g_\perp\right).
\]
Therefore,
\[
\cos\alpha
=
\frac{x_a^\top y}{\|x_a\|\|y\|}
=
\frac{\rho R+G}
{\sqrt{(\rho R+G)^2+\|g_\perp\|^2}}
\]
Now, condition on \(\alpha\), or equivalently on the triple $(R,G,S)$. Since \(x_b\) is an independent isotropic Gaussian vector, its direction
$
V:={x_b}/{\|x_b\|}
$
is uniform on \(S^{d-1}\) and is independent of \((R,G,S)\). 
After conditioning on \(\alpha\), we may place \(x_a\) and \(y\) in the plane spanned by \(e_1,e_2\), so that:
\[
\frac{y}{\|y\|}
=
\cos\alpha\,e_1+\sin\alpha\,e_2.
\]
Therefore, for \(V\sim\mathrm{Unif}(S^{d-1})\),
\[
\cos\theta
=
V^\top e_1,
\quad 
\cos\beta
=
V^\top
(\cos\alpha\,e_1+\sin\alpha\,e_2).
\]
Therefore, we finally get,
\begin{equation*}
\mathcal B_{d,\rho}(N)
=
\mathbb E_{R,G,S}
\left[
\mathbb E_{V\sim\mathrm{Unif}(S^{d-1})}
\left[
e^{-N E_{\mathrm{iid}}(\theta,\alpha,\beta}
\right]
\right],
\end{equation*}
where
$
R\sim\chi_d, \
G\sim\mathcal N(0,1), \ 
S\sim\chi_{d-1}, \
V\sim\mathrm{Unif}(S^{d-1})
$
are independent, and,
\[
\alpha
=
\arccos\left(
\frac{\rho R+G}
{\sqrt{(\rho R+G)^2+S^2}}
\right),
\quad 
\theta=\arccos(V^\top e_1),
\quad 
\beta
=
\arccos\left(
V^\top(\cos\alpha\,e_1+\sin\alpha\,e_2)
\right).
\]
The final expression follows by observing that only the first two components of \(V\) are needed. If we denote 
$
V_1:=V^\top e_1,
\
V_2:=V^\top e_2,
$
then,
$
\cos\theta=V_1,
\
\cos\beta
=
V_1\cos\alpha+V_2\sin\alpha.
$
For $d=2$, $V$ is uniform on the circle, and for \(d>2\), the joint density of the first two components \((V_1,V_2)\) is available in closed form, as, see, e.g.~\cite{diaconis1987dozen},
\begin{equation*}
f_d(v_1,v_2)
=
\frac{d-2}{2\pi}
(1-v_1^2-v_2^2)^{(d-4)/2},
\qquad
v_1^2+v_2^2\le 1. \qed
\end{equation*}

\subsection{Proof of Proposition \ref{prop:angular_floor}}

As described in the main text, the limiting probability of error is given by,
\[
P_e^{\infty}
=
\Pr\left(
\min_{b\neq a}\beta_b\le\alpha
\right).
\]
It remains to evaluate this probability. For all competitors $b \neq a$, define their corresponding directions $
V_b={x_b}/{\|x_b\|},
$
which are independent and uniform on the sphere~
\(S^{d-1}\). Let $V_1:=V^\top e_1$ denote the first component of $V$, let $ \Phi:=\angle(V,e_1) = \arccos(V_1) $ denote the angle between $V$ and $e_1$, and let $F_d(\phi)$ denote its cumulative distribution function,
\[
F_d(\phi)
:=
\Pr\bigl(\arccos V_1\le  \phi \bigr),
\qquad
V\sim\operatorname{Unif}(S^{d-1}).
\]
Then, by rotational invariance we get,
\[
\Pr(\beta_b\le\alpha\mid\alpha)
=
F_d(\alpha),
\]
and since the $V_b$'s are conditionally independent given $\alpha$, we get the order-statistics expression,
\begin{equation}
\label{eq:floor_cap_intermediate}
\Pr\left(
\min_{b\neq a}\beta_b\le\alpha
\right)
=
\mathbb E_\alpha
\left[
1-\bigl(1-F_d(\alpha)\bigr)^{K-1}
\right].
\end{equation}
Now, it remains to obtain an expression for $F_d$.
The first coordinate of a uniformly distributed point on \(S^{d-1}\) has
density,
\[
f_{V_1}(v)
=
\frac{\Gamma(d/2)}
{\sqrt{\pi}\,\Gamma((d-1)/2)}
(1-v^2)^{(d-3)/2},
\qquad -1<v<1,
\]
see, for example, \cite{diaconis1987dozen}.
By defining,
$
W:=\frac{1-V_1}{2},
$
the change of variables formula gives the distribution of $W$,
\[
W\sim
\operatorname{Beta}\left(
\frac{d-1}{2},
\frac{d-1}{2}
\right).
\]
Now,
\[
\angle(V,e_1)\le t
\quad\Longleftrightarrow\quad
V_1\ge\cos t
\quad\Longleftrightarrow\quad
W\le\frac{1-\cos t}{2}.
\]
Since the CDF of the Beta distribution is the regularized incomplete beta function, we get,
\[
F_d( \phi)
=
I_{\frac{1-\cos \phi}{2}}
\left(
\frac{d-1}{2},
\frac{d-1}{2}
\right),
\]
and since,
$
1-I_x(a,a)=I_{1-x}(a,a),
$
we finally get,
\[
1-F_d(\phi)
=
I_{\frac{1+\cos \phi}{2}}
\left(
\frac{d-1}{2},
\frac{d-1}{2}
\right).
\]
Substituting with $\phi= \alpha$ in~\eqref{eq:floor_cap_intermediate} completes the proof, since the unconditional distribution of $\alpha$ is given as before in Proposition~\ref{prop:dim}. \qed



\subsection{Chernoff bound for pseudo-ML decoder}

For one hyperplane \(u\), define,
\[
B(u):=\operatorname{sign}(u^\top y), \quad 
C_m(u):=\operatorname{sign}(u^\top x_m), \quad
p_m(u):=
Q\left(
\frac{|u^\top x_m|}{\sigma}
\right).
\]
The coordinate-wise pseudo-log-likelihood score for candidate \(m\) is,
\[
 \widetilde l_m(u)
=
\mathbf 1\{B(u)=C_m(u)\}\log(1-p_m(u))
+
\mathbf 1\{B(u)\neq C_m(u)\}\log p_m(u).
\]
For a competitor $b \neq a$, define the one-hyperplane score difference,
$
\Delta_{ab}(u)
:=
\widetilde  l _b(u)- \widetilde  l _a(u).
$
Conditional on \(x_a,x_b,z\), the random variables
$
\Delta_{ab}(u_1),\dots,\Delta_{ab}(u_N)
$
are i.i.d. since the hyperplanes \(u_i\) are i.i.d.
For
any \(s \ge 0\), Chernoff's bound gives the pairwise error bound,
\[
\Pr\left(
\widetilde P_b(B)\ge \widetilde P_a(B)
\,\middle|\,
x_a,x_b,z
\right)
\le
\left(
\mathbb E_u
\left[
e^{s\Delta _{ab}(u)}
\,\middle|\,
x_a,x_b,z
\right]
\right)^N.
\]
By defining the geometry-dependent pseudo-ML error exponent,
\begin{equation} \label{eq:epml}
E_{\mathrm{pML}}(x_a,x_b,z)
:=
-\inf_{s\ge0}\log
\mathbb E_u
\left[
e^{s\Delta _{ab}(u)}
\,\middle|\,
x_a,x_b,z
\right].
\end{equation}
a union bound over the competitors yields,
\[
\Pr\left(
E
\,\middle|\,
x_1^K,z
\right)
\le
\sum_{b\neq a}
e^{-N E_{\mathrm{pML}}(x_a,x_b,z)},
\]
and by symmetry, averaging over the random centers $x_1 ^N$ and noise $z$, finally gives,
\[
\Pr (E)
\le
(K-1) \ 
\mathbb E_{x_a,x_b,z}
\left[
e^{-N E_{\mathrm{pML}}(x_a,x_b,z)}
\right].
\]
Hence in this case the error exponent $E_{\mathrm{pML}}(x_a,x_b,z)$ does not reduce to a closed-form
function of the angles \((\theta_b,\alpha,\beta _ b)\), since \(\Delta_{ab}(u)\) also depends
on the realized distances
\(
|u^\top x_a|,
|u^\top x_b|.
\)
So, in order to numerically evaluate the bound, we need a nested MC procedure, which requires to first obtain outer MC samples of the geometry $x_a,x_b,z$, and then for each such realization sample enough hyperplanes to estimate $E_{\mathrm{pML}}(x_a,x_b,z)$ from~(\ref{eq:epml}). This is not particularly helpful, since its computation cost becomes comparable with simulating the whole experiment.

\subsection{Order-statistics error formulas for random hyperplanes}

In order to tighten Chernoff analysis, it is possible to derive analogous order-statistics formulas for the random hyperplanes setting, avoiding the union bound over competitors. The analysis is similar with Section~\ref{s:ord_stats} for orthogonal hyperplanes, but unfortunately, the final expressions no longer simplify enough to be employed in an efficient Monte Carlo estimate. Indeed, computing the resulting estimates would require a computational cost comparable to simulating the whole experiment, meaning that the corresponding expressions are not particularly useful.

\paragraph{Hamming decoder.}
Denoting the conditional pairwise error probability,
\[
p^{\mathrm{Ham}}
:=
\Pr
\left(
d_H(B,C_b)\le d_H(B,C_a)
\,\middle|\,
x_a, z, u_1 ^N
\right),
\]
the exact Hamming error probability is given by,
\[
P_e^{\mathrm{Ham}}
=
\mathbb E_{x_a,z,u_1 ^ N}
\left[
1-
\left(
1-p^{\mathrm{Ham}}
\right)^{K-1}
\right],
\]
but this final expectation does not simplify enough to give an efficient Monte Carlo estimate.

\paragraph{Approximate pseudo-ML decoder.}
Denoting the conditional pairwise error probability,
\[
p^{\mathrm{pML}}
:=
\Pr
\left(
\widetilde P_b(B)\ge \widetilde P_a(B)
\,\middle|\,
x_a,z,u_1^N
\right),
\]
the exact pseudo-ML error probability is,
\[
P_e^{\mathrm{pML}}
=
\mathbb E_{x_a,z,u_1 ^ N}
\left[
1-
\left(
1-p^{\mathrm{pML}}
\right)^{K-1}
\right],
\]
but again this does not lead to an efficient Monte Carlo estimate.

\section{Monte Carlo sampler for Section~\ref{s:ord_ml}} \label{app:sampler}

In the orthogonal hyperplane setting, the ML order-statistic expression
in~(\ref{eq:orth_ml_ord}) can be estimated by Monte Carlo over the scalar score
distributions of \(L_a\) and \(L_b\).
First, generate \(T\) independent true-score samples, $L_a^{(1)},\dots,L_a^{(T)}$. For each 
sample $L_a ^{(t)}$, draw,
\[
Z_i\overset{\mathrm{iid}}{\sim}\mathcal N(0,1),
\qquad
p_i=Q(\rho |Z_i|),
\qquad i=1,\dots,N.
\]
Conditional on \(p_i\), set,
\[
S_i
=
\begin{cases}
\log p_i, & \text{with probability }p_i,\\
\log(1-p_i), & \text{with probability }1-p_i,
\end{cases}
\]
and compute,
\[
L_a ^{(t)}=\sum_{i=1}^N S_i.
\]
Independently, generate $M$ wrong-class score samples $L_b^{(1)},\dots,L_b^{(M)}$.
For
each sample $L_b ^{(m)}$,~draw,
\[
Z_i'\overset{\mathrm{iid}}{\sim}\mathcal N(0,1),
\qquad
q_i=Q(\rho |Z_i'|),
\qquad i=1,\dots,N.
\]
Conditional on \(q_i\), set,
\[
T_i
=
\begin{cases}
\log q_i, & \text{with probability }1/2,\\
\log(1-q_i), & \text{with probability }1/2,
\end{cases}
\]
and compute,
\[
L_b ^{(m)}=\sum_{i=1}^N T_i.
\]
After generating the $M$ wrong-class score samples $L_b^{(1)},\dots,L_b^{(M)}$, sort them.
For each true-score sample \(L_a^{(t)}\), estimate
$
p_t
=
\Pr(L_b\ge L_a^{(t)})
$,
by the empirical tail fraction,
\[
\widehat p_t
=
\frac1M
\sum_{\ell=1}^M
\mathbf 1\{L_b^{(\ell)}\ge L_a^{(t)}\}.
\]
The order-statistic Monte Carlo error estimate can then be computed as,
\[
\widehat P_e^{\mathrm{ML,orth}}
=
\frac1T
\sum_{t=1}^T
\left[
1-
(1-\widehat p_t)^{K-1}
\right].
\]

\section{Simulation results} \label{app:sims}

\subsection{Orthogonal hyperplanes with noise}

\begin{figure}[h!]
    \centering
    \includegraphics[width=0.48\linewidth]{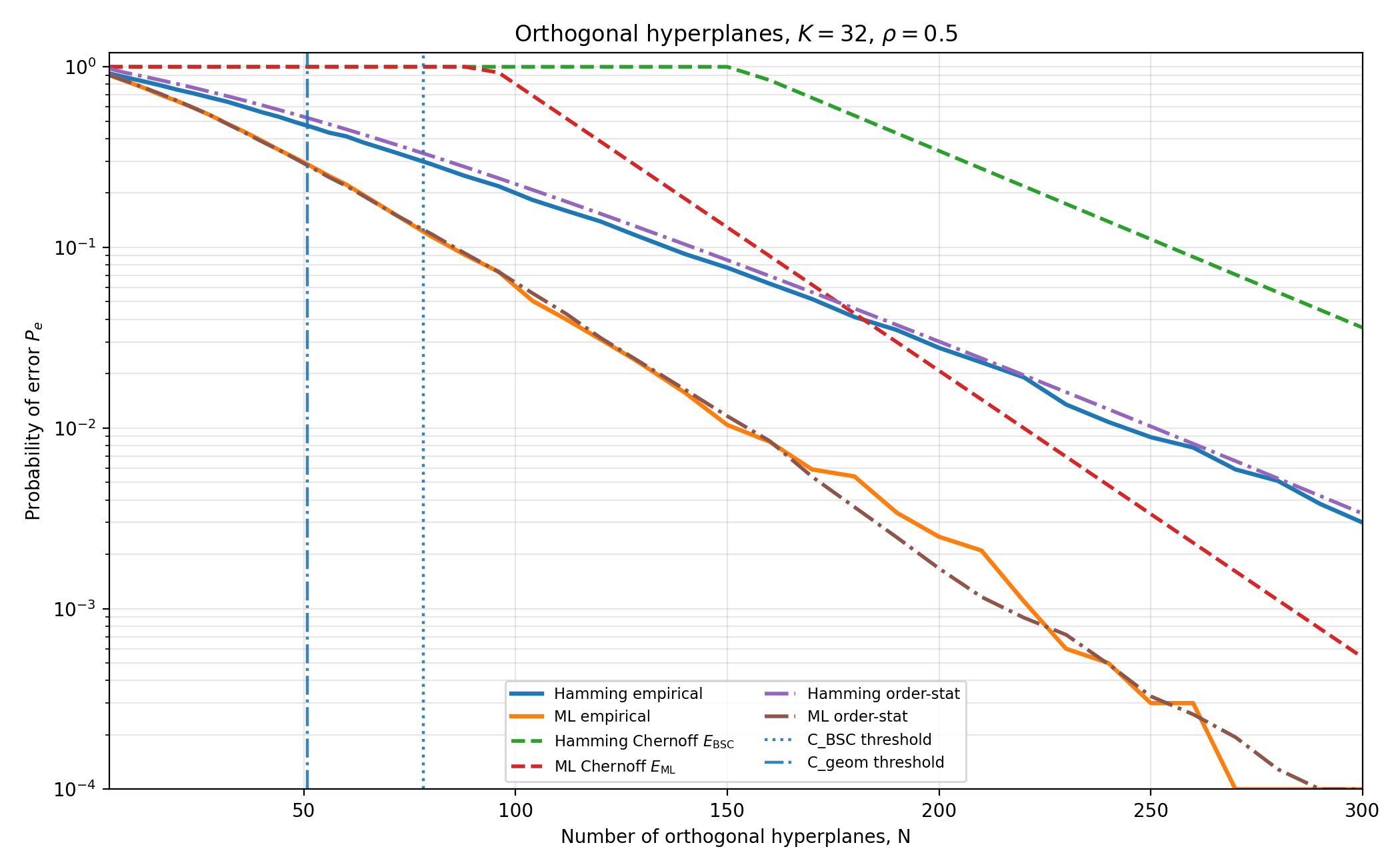}
    ~
        \includegraphics[width=0.48\linewidth]{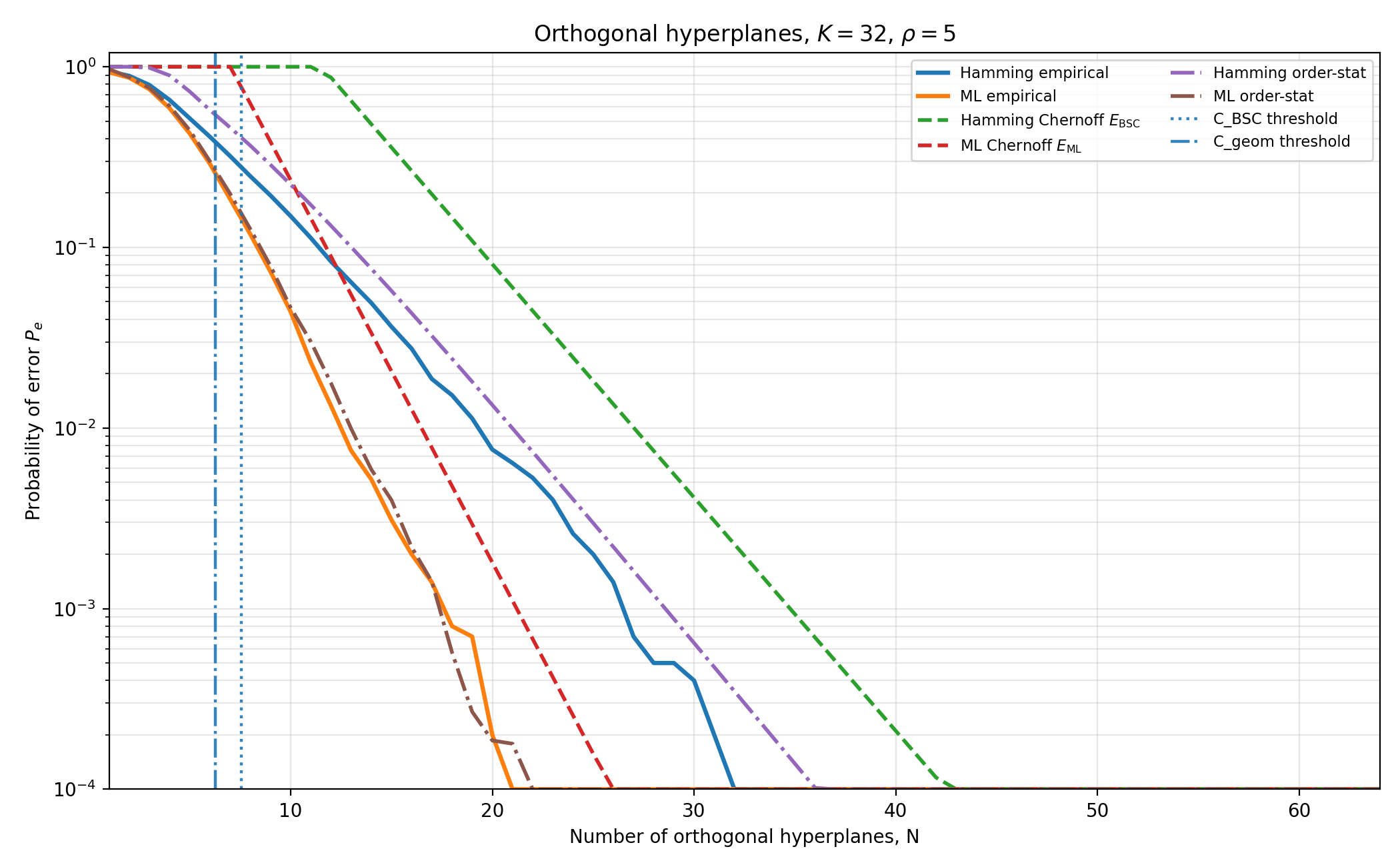}
    ~
        \includegraphics[width=0.48\linewidth]{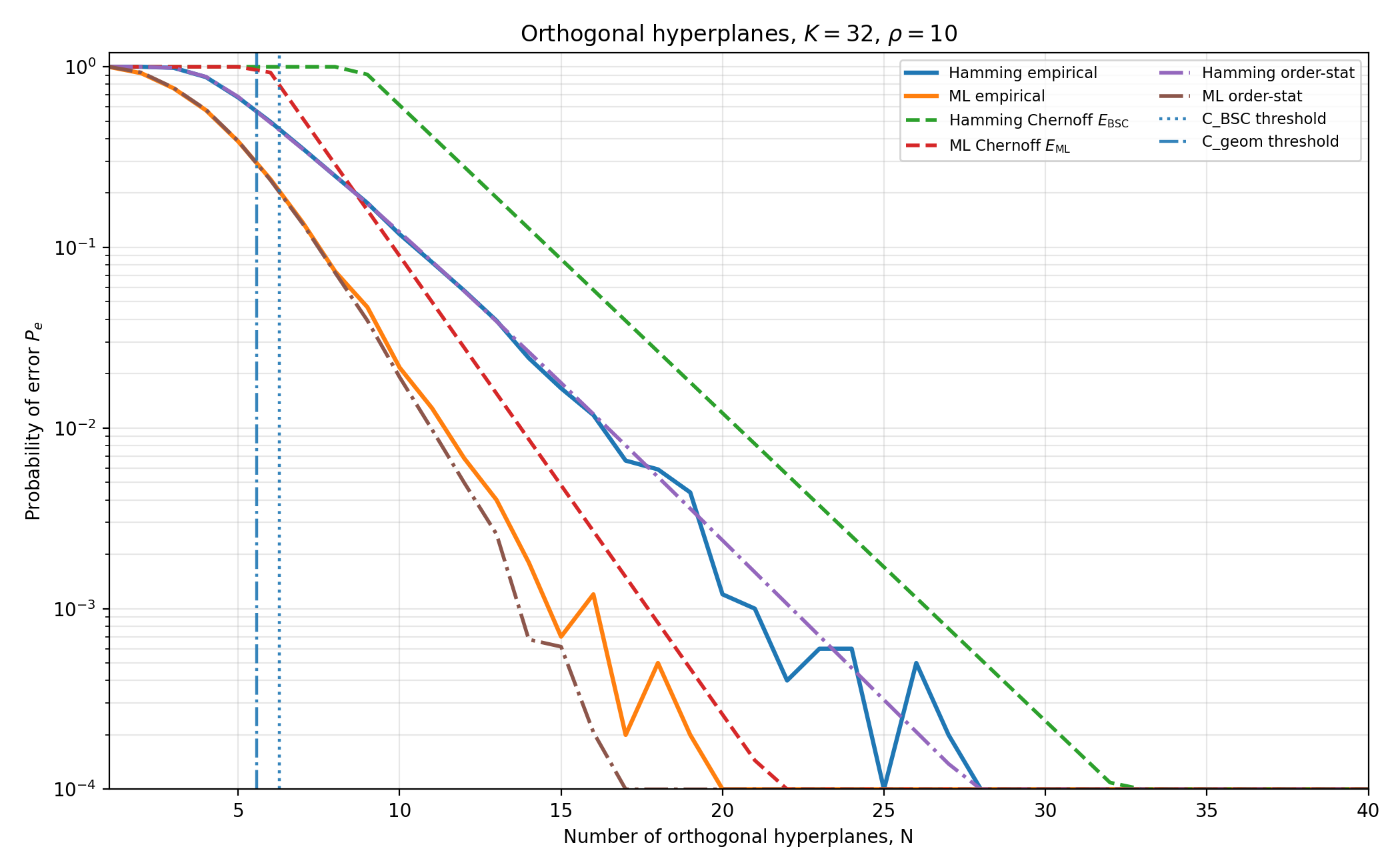}
    \caption{Probability of error as a function of the number of orthogonal hyperplanes. Top Left: $\rho=0.5$,  Top Right: $\rho=5$, Bottom: $\rho = 10$.}
    \label{fig:orth_app}
\end{figure}

\subsection{Random hyperplanes with noise}

\begin{figure}[h!]
    \centering
    \includegraphics[width=0.48\linewidth]{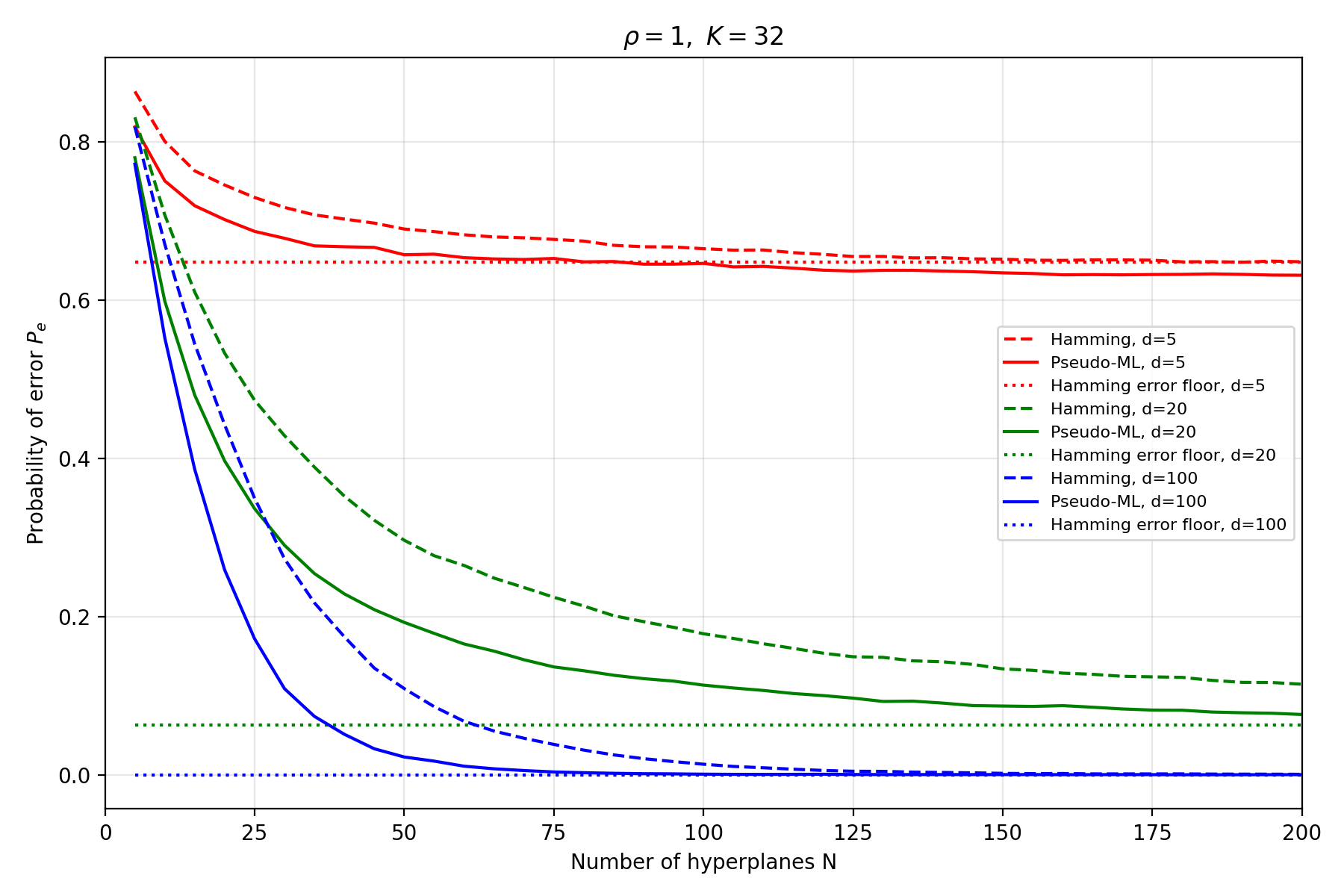}
    ~
        \includegraphics[width=0.48\linewidth]{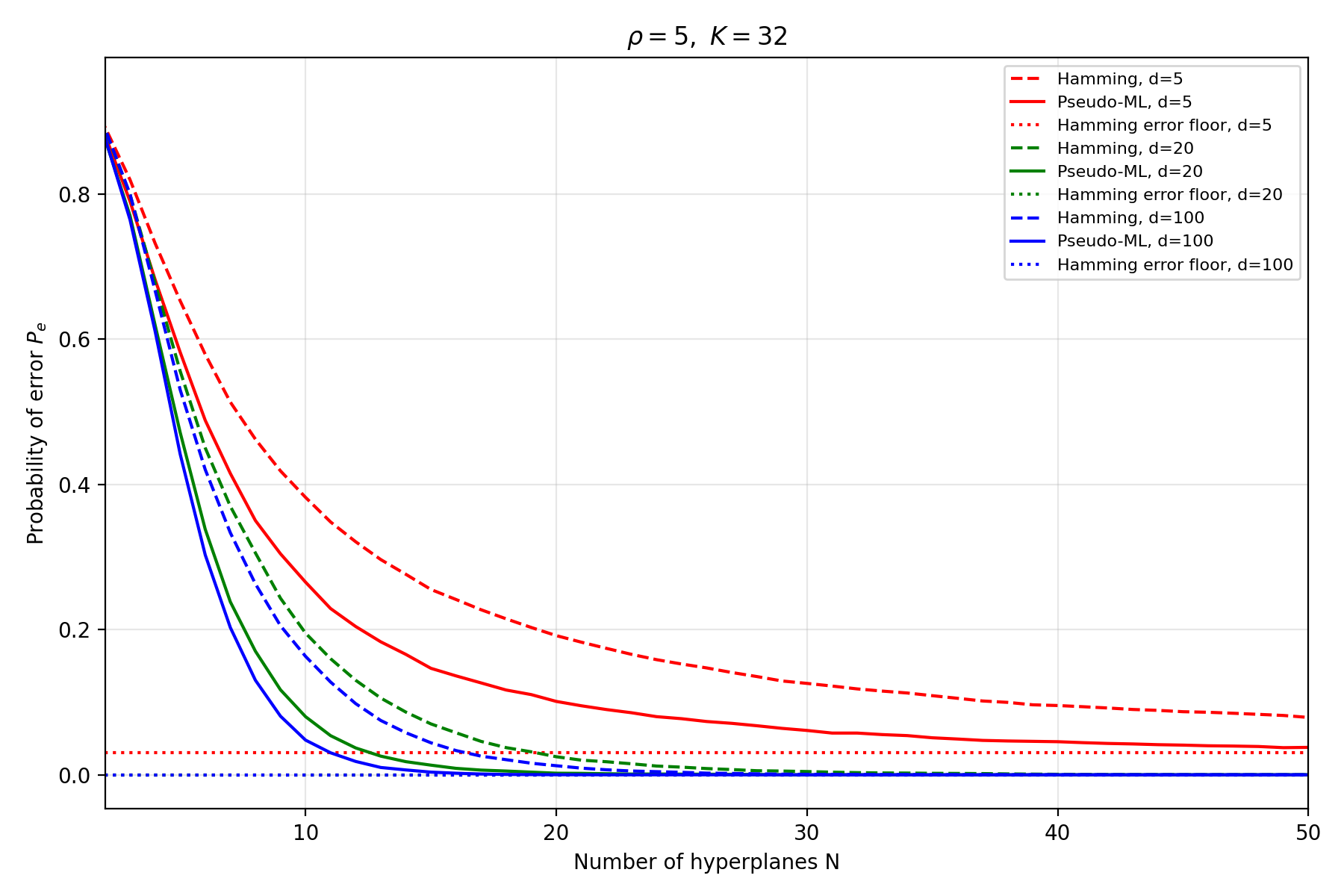}
    \caption{Probability of error as a function of the number of random  hyperplanes.  Left: $\rho=1$,  Right: $\rho=5$.}
    \label{fig:iid_app}
\end{figure}

\end{document}